\documentclass[compsoc,conference,a4paper,10pt,times]{IEEEtran}
\usepackage[utf8]{inputenc}

\usepackage{cite}
\usepackage{centernot}
\usepackage{xcolor}
\definecolor{teal}{HTML}{008080}
\usepackage[colorlinks=true,urlcolor=black]{hyperref}
\hypersetup{
   colorlinks=true,
   citecolor=teal
}

\usepackage{graphicx}

\usepackage{multirow}

\usepackage{enumitem}

\renewcommand{\theenumii}{\theenumi.\arabic{enumii}.}

\usepackage{amsmath, amssymb, amsthm}
\usepackage[font=bf]{caption}
\usepackage[caption=false, font=footnotesize]{subfig}

\newcommand{\sample}[1]{\mathrel{{\leftarrow}\vcenter{\hbox{\scriptsize\rmfamily\upshape\ensuremath{{#1}}}}}}

\newcommand{\descr}[1]{\vspace{0.2cm} \noindent \textbf{#1}}

\newcommand{\adversary}{\mathcal{A}}

\newtheorem{theorem}{Theorem}
\newtheorem{proposition}{Proposition}
\newtheorem{observation}{Observation}
\newtheorem{corollary}{Corollary}
\theoremstyle{definition}
\newtheorem{experiment}{Experiment}
\theoremstyle{definition}
\newtheorem{definition}{Definition}

\allowdisplaybreaks

\newcommand{\blue}[1]{#1}

\newif\iffull
\fulltrue 

\title{On the (In)Feasibility of Attribute Inference Attacks on Machine Learning Models}

\author{
\IEEEauthorblockN{
Benjamin Zi Hao Zhao\IEEEauthorrefmark{3}\IEEEauthorrefmark{2},
Aviral Agrawal\IEEEauthorrefmark{4}\IEEEauthorrefmark{1}\IEEEauthorrefmark{2}, 
Catisha Coburn\IEEEauthorrefmark{5},
Hassan Jameel Asghar\IEEEauthorrefmark{1}\IEEEauthorrefmark{2}, 
}\IEEEauthorblockN{
Raghav Bhaskar\IEEEauthorrefmark{2},
Mohamed Ali Kaafar\IEEEauthorrefmark{1}\IEEEauthorrefmark{2},
Darren Webb\IEEEauthorrefmark{5}, and
Peter Dickinson\IEEEauthorrefmark{5}
}

\IEEEauthorblockA{ \IEEEauthorrefmark{1}Macquarie University,
\IEEEauthorrefmark{3}University of New South Wales,
\IEEEauthorrefmark{2}Data61-CSIRO,
\IEEEauthorrefmark{4}BITS Pilani K.K.Birla Goa campus,} 
\IEEEauthorblockA{\IEEEauthorrefmark{5}Cyber \& Electronic Warfare Division, Defence Science and Technology Group, Australia}

}

\usepackage{fancyhdr}

\begin{document}
\pagestyle{fancy}
\maketitle

\begin{abstract}
With an increase in low-cost machine learning APIs, advanced machine learning models may be trained on private datasets and monetized by providing them as a service. However, privacy researchers have demonstrated that these models may leak information about records in the training dataset via membership inference attacks. In this paper, we take a closer look at another inference attack reported in literature, called attribute inference, whereby an attacker tries to infer missing attributes of a partially known record used in the training dataset by accessing the machine learning model as an API. We show that even if a classification model succumbs to membership inference attacks, it is unlikely to be susceptible to attribute inference attacks. We demonstrate that this is because membership inference attacks fail to distinguish a member from a nearby non-member. We call the ability of an attacker to distinguish the two (similar) vectors as strong membership inference. We show that membership inference attacks cannot infer membership in this strong setting, and hence inferring attributes is infeasible. However, under a relaxed notion of attribute inference, called approximate attribute inference, we show that it is possible to infer attributes close to the true attributes. We verify our results on three publicly available datasets, five membership, and three attribute inference attacks reported in literature.
\end{abstract}

\section{Introduction}
\label{sec:intro}
The introduction of low-cost machine learning APIs from Google, Microsoft, Amazon, IBM, etc., has enabled many companies to monetize advanced machine learning models trained on private datasets by exposing them as a service. This also caught the attention privacy researchers who have shown that these models may leak information about the records in the training dataset via membership inference (MI) attacks. In an MI attack, the adversary (for instance, a user of the service) with API access to the model, can use the model's responses (class labels and probability/confidence of each label) on input records of his/her choice to infer whether a target input was part of the training dataset or not. This can be a serious privacy breach when the underlying dataset is sensitive, e.g., medical data, \blue{mobility traces and financial transactions~\cite{shokri-mia, ml-leaks}.} 

To date, membership inference attacks have been the primary focus of studies that have contemplated on traits of the datasets and machine learning models that impact the attacks' likelihood and accuracy~\cite{shokri-mia, ml-leaks, somesh-overfit, nasr2018comprehensive, disparate-sub-group}. Our focus is on a related, and perhaps a more likely attack in practice, where the adversary with partial background knowledge of a target’s record seeks to complete its knowledge of the missing attributes by observing the model’s responses. This attack is called \emph{model inversion}~\cite{model-inversion,pharma}, or in general \emph{attribute inference} (AI)~\cite{somesh-overfit}. Yeom et al.~\cite{somesh-overfit} provide a formal definition of an AI adversary, and argue that this adversary can infer the missing attribute values by using an MI adversary as a subroutine. More precisely, for a missing attribute with $t$ possible values, the AI adversary constructs $t$ different input (feature) vectors, gives them as input to the MI adversary, and outputs the attribute value which corresponds to the vector that the MI adversary deems to be in the training dataset. 

Beyond providing a formal definition, Yeom et al. experimentally validate the success of an AI attack on regression models, and conclude that the more overfit the model, the higher the success of the AI attack~\cite[\S 6.3]{somesh-overfit}. Seeking to replicate their results on classification models (rather than regression models), where the adversary is given a partial record and its true label, our results in this paper turn out to be different. We show that even if the target classification model is susceptible to MI attacks, AI attacks on the same model have negligible advantage. Furthermore, the results persist even for highly overfitted models. We explore the reasons behind this failure, and find that in order for AI attacks to be successful, the underlying MI attack, used as a subroutine, should be able to infer membership in a stronger sense. More precisely, the MI attack should be able to distinguish between a member of the training dataset and any non-members that are \emph{close} to that member, according to a suitable distance metric 
(we consider several such distance metrics based on the nature of the dataset).
We call this, strong membership inference (SMI), parameterized by the distance from the training dataset. 

We formulate the notion of SMI, and prove that a successful MI attack does not necessarily mean a successful SMI attack. Furthermore, we also formally show that a successful SMI attack is essential for an AI attack. This result implies that even a standalone AI attack, which does not use an MI attack as a subroutine, is bound to fail if SMI attacks are unsuccessful. We experimentally validate these results by evaluating several proposed MI attacks from the literature on several discrete and continuous datasets, and target machine learning models, and show that while these attacks are successful in inferring membership, they fall well short as an SMI attack, and consequently as an AI attack. \blue{On the positive side (from an attacker's point of view), we investigate a more relaxed notion of attribute inference, called \emph{approximate attribute inference} (AAI), where the adversary is only tasked with finding attributes \emph{close} to the target attributes, according to a given distance metric. We show that while AI attacks are not applicable, AAI attacks perform significantly better, and improve as the target model becomes more overfit. The AAI notion is also a natural extension of the (exact) AI notion for continuous attributes which has mostly been used in discrete settings~\cite{somesh-overfit}.} 

In more detail, our main contributions are as follows.
\begin{itemize}[noitemsep, topsep=0pt, leftmargin=*]
    \item We provide a formal treatment of \blue{membership, attribute, and approximate attribute inference attacks, and propose a new definition of strong membership inference (SMI)}, building on the work from~\cite{somesh-overfit} on the definitions of MI and AI in Section~\ref{sec:formal}. We formally prove that an SMI adversary is \textit{strictly stronger} than an MI adversary (Theorem~\ref{the:mi-smi}), and that SMI is necessary for AI (Theorem~\ref{the:nosmi-noai}). 
    \item We experimentally validate our theoretical findings through an extensive set of experiments involving five MI attacks, three black-box and two white-box, from ~\cite{shokri-mia, somesh-overfit, ml-leaks, nasr2018comprehensive}, eight datasets (constructed from 3 main binary and continuous datasets), and several target machine learning models (neural networks, support vector machines, logistic regression, and random forests) (cf. Section~\ref{sec:experiments}). Our results in Section~\ref{sec:mia-exp} validate our formal separation and show that while these attacks are successful to infer membership, they are ineffective in inferring membership at distances close to the training dataset (SMI).
    
    \item  In Section~\ref{sec:aia-exp}, we further construct 3 AI attacks using the MI attacks of \cite{shokri-mia, somesh-overfit} and \cite{ml-leaks} as a subroutine, and show via experiments that these attacks are not effective in inferring attributes, even if we increase the overfitting levels of the target model. On the other hand, we show that our constructed AI attacks can approximately infer attributes (AAI), with the advantage increasing as the level of overfit of the target model increases. 
    \item Our other key findings include explanation behind the seemingly contradictory conclusions about AI attacks on regression models~\cite{somesh-overfit} and classification models (our focus) in Section~\ref{sec:ai-result}. We also show that the success of an MI attack is dependent on the class label of the vector; if the corresponding class occupies an overwhelmingly large portion of the feature space, then training records belonging to this class are harder to distinguish from non-members (cf. Section~\ref{sec:mia-labels}). This gives one plausible reason why MI attacks have always performed poorly on target models for binary classification problems~\cite{ml-leaks, shokri-mia}. 

\end{itemize}

\section{Formal Treatment of Membership and Attribute Inference Attacks}
\label{sec:formal}


In this section, we formally introduce the privacy notions of strong membership inference (SMI), and \blue{recap the notions of membership, attribute and approximate attribute inference}. In order to define them, we need rigorous definitions of a distance metric on the feature space, missing (features) attributes of a feature vector and its relation to distance, and how the probability distribution on the feature space behaves around feature vectors. We first define these concepts in the next section followed by privacy definitions in Section~\ref{sec:definition}.

\subsection{Notation and Definitions}

\descr{Feature Space.} Let $\mathbb{D}$ denote a subset of the real space $\mathbb{R}$. We assume the feature space to be $\mathbb{D}^m$, where each point $\mathbf{x} \in \mathbb{D}^m$ is called a feature vector consisting of $m$ elements/features. We assume the output space to be $Y = \mathbb{R}^*$. Let $\mathcal{D}$ be a distribution over $\mathbb{D}^m$. The \emph{training} dataset $X$ is defined as a multiset of $n$ elements drawn i.i.d. from $\mathbb{D}^m$ with distribution $\mathcal{D}$. Each $\mathbf{x} \in X$ is accompanied by its \emph{true} label $\mathbf{y} \in Y$. We denote this mapping by $c$, which we call the \emph{target concept} following standard terminology~\cite{learning-theory, shai-shai}. Thus, for each $\mathbf{x} \in X$, $c(\mathbf{x})$ denotes is true label. The term label is used generically; it may be discrete, denoting different classes, or it may be continuous, denoting the confidence or probability score for the different classes. The support of distribution $\mathcal{D}$ is defined as $\text{supp}(\mathcal{D}) = \{ \mathbf{x} \in \mathbb{D}^{m} \mid p_{\mathbf{x}} > 0 \}$,
where $p_{\mathbf{x}}$ is ${\Pr}_{\mathcal{D}}(\mathbf{x})$ if $\mathbb{D}^m$ is discrete and $f_\mathcal{D}(\mathbf{x})$ if $\mathbb{D}^m$ is continuous, $f$ being the probability density function. 
The notation $a \sample{\$} A$ indicates sampling an element $a$ from some set $A$ uniformly at random. The notation $\mathbf{x} \leftarrow \mathcal{D}$ denotes sampling a feature vector according to the distribution $\mathcal{D}$. Similarly, the notation $X \leftarrow \mathcal{D}^n$ denotes sampling a multiset of $n$ feature vectors (training set) drawn i.i.d. from $\mathcal{D}$. 



\descr{Machine Learning Models.} A machine learning model $h_X$ trained on $X$, takes as input $\mathbf{x} \in \mathbb{D}^m$ and outputs a label $\mathbf{y} \in Y$. Let $L : Y \times Y \rightarrow \mathbb{R}$ denote a loss function. The training loss of $h$, denoted, $L_{\text{tr}}(h)$, determines how much $h$ differs from $c$ on all $\mathbf{x} \in X$. 
Similarly we define the test loss of $h$ by $L_{\text{test}}(h)$, which is evaluated by computing $h(\mathbf{x})$ and $c(\mathbf{x})$ over the distribution $\mathcal{D}$.
For instance, if $Y$ is discrete, then $L$ can be the $0$-$1$ loss function, which evaluates to $L(h(\mathbf{x}), c(\mathbf{x})) = 0$, if $h(\mathbf{x}) = c(\mathbf{x})$, and $1$ otherwise~\cite{somesh-overfit}. The generalization error of $h$ is defined as
\begin{equation}
\label{eq:gen-err}
\text{err}(h) =  L_{\text{tr}}(h) - L_{\text{test}}(h). \end{equation}

The exact form of the loss function $L$ depends on the learning problem. More specifically, it depends on the nature of $Y$. If the learning problem is that of classification among $k$ different classes, which is our focus, we have $|Y| = k$. The true label of a sample $\mathbf{x}$ is then a $k$-element vector $\mathbf{y} \in Y$ with $1$ in the position corresponding to the true class, and 0 in all other places. A classifier $h_X$ however, may output a vector $\mathbf{y}' \in Y$ such that each element $y_i \in [0, 1]$ and $\lVert \mathbf{y}' \rVert_1 = 1$.  
 
\descr{Metrics.} The notions of SMI and AAI, informally introduced in the introduction, are based on the ability to distinguish nearby vectors in the feature space. The notion of ``nearness'' 
is based on a distance metric on the feature space $\mathbb{D}^m$. The examples of metrics used in this paper are Hamming distance $d_H$ for binary datasets, i.e., over the domain $\mathbb{D}^m = \{0, 1\}^m$, and Manhattan distance $d_M$ for normalized continuous datasets, i.e., over $\mathbb{D}^m = [-1, 1]^m$. In general, our results generalize to any \emph{conserving} metric 
\iffull
(See Appendix~\ref{app:metrics}).
\fi
The following defines the distance of a non-member vector from the training dataset.


\begin{definition}[Distance and Neighbors] 
\label{def:dist}
Let $d$ be a (conserving) metric on $\mathbb{D}^m$. Let $r$ be a positive real number and let $\mathbf{x} \in \mathbb{D}^m$. The set of $r$-neighbors of $\mathbf{x}$ is the $r$-ball centered at $\mathbf{x}$ defined as
\[
B_d(\mathbf{x}, r) = \{ \mathbf{x}' \in \mathbb{D}^m \mid d(\mathbf{x}, \mathbf{x}') \le r \}.
\]
A member of $B_d(\mathbf{x}, r)$ is called an $r$-neighbor of $\mathbf{x}$. The distance of a vector $\mathbf{x} \in \mathbb{D}^m$ from a set $X \subseteq \mathbb{D}^m$ is defined as $\min_{\mathbf{x}' \in X} d(\mathbf{x}, \mathbf{x}')$. 
We call $\mathbf{x}'$ the nearest neighbor of $\mathbf{x}$ in $X$. 
\qed
\end{definition}

For attribute inference, we define the notion of a vector with missing attributes as \emph{portion}:

\begin{definition}[Portions]
\label{def:portions}
We introduce a special symbol $*$ called star, and define $\mathbb{D}^* = \mathbb{D} \cup \{*\}$. Let $S$ be a subset of indexes from $[m]$, which we call the set of unknown features. We define the map $\phi_S : \mathbb{D}^m \rightarrow {\mathbb{D}^*}^m$, which given as input a feature vector $\mathbf{x}$ outputs a vector $\mathbf{x}^*$, such that $x^*_i = *$ for each $i \in S$ and $x^*_i = x_i$ for all $i \notin S$. We call $\mathbf{x}^* = \phi_S(\mathbf{x})$ a \emph{portion} of $\mathbf{x}$ under $S$, or simply a portion of $\mathbf{x}$ if reference to the set $S$ is not relevant. The set of features that are \emph{masked}, i.e., replaced by $*$, in $\phi_S(\mathbf{x})$ will be called the \emph{unknown part} of $\mathbf{x}^*$. \qed
\end{definition}

\begin{definition}[Siblings]
\label{def:siblings}
 Define the set:
\[
\Phi_S(\mathbf{x}) = \{ \mathbf{x}' \in \mathbb{D}^m \mid \phi_S(\mathbf{x}) = \phi_S(\mathbf{x}') \},
\]
then $\Phi_S(\mathbf{x})$ is called the set of siblings of $\mathbf{x}$ under $S$, and any member of the set a sibling of $\mathbf{x}$ under $S$. Note that $\mathbf{x}$ is also a sibling of itself. 
\end{definition}

For attribute inference, the algorithm will be given a portion $\mathbf{x}^* = \phi_S(\mathbf{x})$, such that the feature corresponding to the set $S$ will be missing (unknown). The set $\Phi_S(\mathbf{x})$ contains all vectors which could possibly have the portion $\mathbf{x}^*$, including the original vector $\mathbf{x}$. These are the possible \emph{candidates} of the portion, and the algorithm would need to distinguish them from $\mathbf{x}$. 
In Appendix~\ref{app:metrics}, we show that given a vector $\mathbf{x}$, all of its possible portions with $i$ unknown features are within a ball whose radius can be determined through $i$. This result is useful to show the link between attribute inference and strong membership inference, as we shall see later. 

In some of our inference definitions, we would need to sample vectors in the vicinity of some feature vector $\mathbf{x}$. Depending on the distribution $\mathcal{D}$, it may well be the case that the vectors around $\mathbf{x}$ have a negligible probability of being sampled as feature vectors. Thus, the adversary may simply be able to infer non-membership by checking which vector is not likely to be sampled under $\mathcal{D}$~\cite{somesh-overfit}. To overcome this technical issue, we assume that the distribution $\mathcal{D}$ is such that there is at least one vector within a small radius around $\mathbf{x}$ which is assigned a similar probability as $\mathbf{x}$. This is made precise by the following definitions.  




\begin{definition}[Induced Distribution] 
\label{def:ind-dist}
Let $Z$ be a set of feature vectors. Define $Z_\mathcal{D} = \text{supp}(\mathcal{D}) \cap Z$. We say that a vector $\mathbf{z}$ is sampled from $Z$ according to the distribution induced by $\mathcal{D}$ if the resulting random variable has probability mass function $\frac{p_\mathbf{z}}{\sum_{\mathbf{z}' \in Z_\mathcal{D}} p_{\mathbf{z}'}}$
or the probability density function $\frac{p_\mathbf{z}}{\int_{Z_\mathcal{D}} f_\mathcal{D}(\mathbf{z}')d\mathbf{z}'}$
in the continuous case. \qed
\end{definition}

Note that the probabilities are only defined if $Z_\mathcal{D}$ is non-empty. We shall always assume this to be the case.

\begin{definition}[Indistinguishable Neighbor Assumption]
\label{def:smooth-dist}
Let $r > 0$, and let $d$ be a metric. Let $\mathbf{x} \sample{} \mathcal{D}$. Let \blue{$\mathbf{x}'$ be sampled from $B_d(\mathbf{x}, r)$ according to the distribution induced by $\mathcal{D}$. Let $\mathcal{A}$ be any algorithm (distinguisher) taking as input a feature vector $\mathbf{x}$ and a distribution $\mathcal{D}$, which outputs 1 if $\mathbf{x} \sample{} \mathcal{D}$ and $0$, otherwise. Let $b \sample{\$} \{0, 1\}$. Let $\mathcal{A}$ be given $\mathbf{x}$, if $b = 1$ and $\mathbf{x}'$, if $b = 0$. Then 
\begin{equation}
\label{eq:r-neigh}
\Pr[ \mathcal{A}(\mathbf{x}, \mathcal{D}) = 1] - \Pr[\mathcal{A}(\mathbf{x}', \mathcal{D}) = 1 ] \leq \epsilon(r).    
\end{equation}
We call $\epsilon(r)$, the $r$-neighbor distinguishability advantage, and assume it to be negligible for small $r$.} 
\qed


\end{definition}

The above assumption states around any vector $\mathbf{x}$, there are some vectors sampled according to the distribution induced by $\mathcal{D}$ that are indistinguishable from $\mathbf{x}$ under $\mathcal{D}$. 
Note that this does not apply to all neighbors of $\mathbf{x}$ \blue{(which may be out of distribution)}. \blue{Put in other words, it states that around any vector $\mathbf{x}$, there are neighborhood vectors which have similar probability of being sampled under $\mathcal{D}$. It is easy to see why this assumption should hold on datasets with continuous attributes, as minor changes in the attributes would hardly be off-distribution.} We argue that this is \blue{also a plausible assumption for discrete datasets}. For instance, consider the Purchase (shopping transactions) dataset~\cite{purchase_dataset}, which records the items bought by customers; 1 if the corresponding item is purchased by the customer and 0, otherwise. Given any vector $\mathbf{x}$, a nearby vector where a few item purchases have been removed can barely be considered an anomaly. Further note that the ability to distinguish increases, the further we move from the original vector, since now there are other vectors likely to be sampled through the induced distribution which are starkly different from $\mathbf{x}$, i.e., at greater distance from $\mathbf{x}$. Hence, the advantage $\epsilon(r)$ is defined as a function of $r$. \blue{To experimentally validate our claim, we trained a generative adversarial network (GAN) on the Purchase dataset to see if it can distinguish between original and nearby vectors. The results shown in Appendix~\ref{app:valid} are in agreement with our assumption.} 


\descr{Decision Regions.} Our final definition in this section is that of decision regions, i.e., regions in the feature space assigned to a given class. We shall show later that performance of membership inference is linked to the volume of decision regions. Let $k \ge 2$ be the number of classes. 
\begin{definition}
\label{def:dr}
Given a classifier $h_X$, for each class $j \in [k]$, we define its \emph{decision region} (DR) as
\begin{equation}
\label{eq:dr}
\mathcal{R}_j = \{ \mathbf{x} \in \mathbb{D}^m : h_X(\mathbf{x}) = j \}
\end{equation}
\end{definition}
This is analogous to the definition of acceptance region in~\cite{zhao2020resilience}. Similar to~\cite{zhao2020resilience}, we sample a large number of feature vectors from $\mathbb{D}^m$ uniformly at random, and use the fraction of vectors labelled $j$ by $h_X$ to estimate the \emph{fractional volume} of the decision region $\mathcal{R}_j$. Overloading notation, we shall use decision region to mean both the region and its fractional volume. A class is said to \emph{dominate} another class if the DR of the former is larger than the DR of the latter. The class with the largest DR shall be called the \emph{most dominant} class.

\subsection{Formal Results: Relationship between Variants of Membership and Attribute Inference}
\label{sec:definition}

\descr{Membership Inference.} Our first definition is that of membership inference which is derived from the definition in~\cite{somesh-overfit}. 

\begin{experiment}[Membership Inference (MI)~\cite{somesh-overfit}]
\label{exp:mem-inf-somesh}
Let $\adversary$ be the adversary, let $X \leftarrow \mathcal{D}^n$ be the input dataset.
\begin{enumerate}[itemsep=0pt]
    \item Construct model $h_{X}$.
    \item Sample $b \sample{\$} \{0, 1\}$.
    \item If $b = 0$, sample $\mathbf{x} \sample{} \mathcal{D}$.
    \item Else if $b = 1$, sample $\mathbf{x} \sample{\$} X$.
    \item $\adversary$ receives $\mathbf{x}$, $c(\mathbf{x})$ and oracle access to $h_{X}$.
    \item $\adversary$ announces $b' \in \{0, 1\}$. If $b' = b$, output 1, else output 0.
\end{enumerate}
\end{experiment}


\descr{Using the True Label.} Note that in addition to the vector $\mathbf{x}$,
its true label $c(\mathbf{x})$ is also given to the adversary. This then allows the adversary to compute the loss function $L(h_X(\mathbf{x}), c(\mathbf{x}))$ from the output of the model $h_X$. This is considered for instance in \cite{somesh-overfit}, the shadow model technique in~\cite{shokri-mia} and the shadow model variants of membership inference attacks in~\cite{ml-leaks}. However, note that the true label is not necessarily required as is demonstrated in one of the attacks in~\cite{ml-leaks} which only uses the knowledge of the input sample and the prediction returned by $h_X$. In this case, the adversary simply ignores the true label $c(\mathbf{x})$. The same is true in all the other experiments (definitions) to follow. 


Let $\text{Exp}_{\text{MI}}(\adversary, h, n, \mathcal{D})$ denote the output of the above experiment. 

\begin{definition}[Membership Inference Advantage]
\label{def:mi-adv}
The membership inference advantage of $\adversary$ on the classifier $h$, i.e., $\text{Adv}_{\text{MI}}(\adversary, h, n, \mathcal{D})$, is defined as
\begin{align*}
   &\quad \Pr [b' = 1 \mid b = 1] - \Pr [b' = 1 \mid b = 0] \\
   &= \Pr [b' = 0 \mid b = 0] - \Pr [b' = 0 \mid b = 1]
\end{align*}
\end{definition}

It is the thesis of this paper that an MI adversary with a significant advantage in distinguishing between members and non-members is due to the fact that non-members are at a significant distance away from member vectors. If on the other hand a non-member vector is close to a member vector, then the adversary may not be able to distinguish between the two. We therefore present another definition of membership inference, called strong membership inference (SMI) defined next. The definition challenges the adversary to distinguish between two neighboring feature vectors. The closeness of the two vectors is controlled by the parameter $r$ in the definition. We show later why such a strong inference attacker is a better starting point for constructing an attribute inference attacker in the spirit of~\cite{somesh-overfit}. 

\begin{experiment}[$r$-Strong Membership Inference (SMI)]
\label{exp:strong-mem-inf}

Let $\adversary$ be the adversary, let $X \leftarrow \mathcal{D}^n$ be the input dataset, let $d$ be a (conserving) metric, and let $r > 0$ be a real number.
\begin{enumerate}
    \item Construct model $h_X$.
  \item Sample $b \sample{\$} \{0, 1\}$.
    \item Sample $\mathbf{x}_0 \sample{\$} X$.
    \item If $b = 0$, sample $\mathbf{x}$ from $B_d(\mathbf{x}_0, r)$ according to the distribution induced by $\mathcal{D}$ (cf. Definition~\ref{def:ind-dist}).
    \item Else if $b = 1$, $\mathbf{x} = \mathbf{x}_0$.
    \item $\adversary$ receives $\mathbf{x}$, $c(\mathbf{x})$ and oracle access to $h_{X}$.
    \item $\adversary$ announces $b' \in \{0, 1\}$. If $b' = b$, output 1, else output 0.
\end{enumerate}
\end{experiment}


\begin{definition}[Strong Membership Inference Advantage]
\label{def:smi-adv}
The SMI advantage of $\adversary$ on the classifier $h$, i.e., $\text{Adv}_{\text{SMI}}(\adversary,
   h, r, n, \mathcal{D})$, is defined as
\begin{align*}
   &\quad \Pr [b' = 1 \mid b = 1] - \Pr [b' = 1 \mid b = 0] \\
   &= \Pr [b' = 0 \mid b = 0] - \Pr [b' = 0 \mid b = 1] 
\end{align*}
\end{definition}




\descr{Relationship between MI and SMI.}
SMI is the same as MI if $r$ is large enough to encompass all feature vectors in the support of $\mathcal{D}$. Otherwise, the next theorem shows that the two definitions are not equivalent. 

\begin{theorem}
\label{the:mi-smi}
There exists a domain $\mathbb{D}^m$, a distribution $\mathcal{D}$ on the domain, an $r > 0$, a dataset $X \leftarrow \mathcal{D}^n$, a classifier $h$, and an algorithm $\mathcal{A}$ such that an MI adversary gains non-negligible advantage using $\mathcal{A}$ whereas an SMI adversary has 0 advantage using the same algorithm.
\end{theorem}

\begin{proof}
\iffull
See Appendix~\ref{app:inf-relations}.
\fi
\end{proof}

The proof of the above result essentially constructs a dataset such that the output of the classifier is constant around any vector $\mathbf{x}$ in the dataset. In a real-world dataset, this implies that we assume the output of the classifier to be nearly constant around any feature vector $\mathbf{x}$, thus making it hard for an SMI attack to distinguish non-members in the vicinity of members. We shall later show that this assumption holds for real-world datasets and classifiers.

\descr{Attribute Inference.} We first start with the definition of attribute inference derived from~\cite{somesh-overfit}. 

\begin{experiment}[Attribute Inference (AI)~\cite{somesh-overfit}]
\label{exp:attr-inf-somesh}
Let $\adversary$ be the adversary, let $X \leftarrow \mathcal{D}^n$ be the input dataset, and let $S$ be a subset of $[m]$ with cardinality $m'$ such that $1 \le m' < m$.
\begin{enumerate}
    \item Construct model $h_{X}$.
    \item Sample $b \sample{\$} \{0, 1\}$.
    \item If $b = 0$, sample $\mathbf{x} \sample{} \mathcal{D}$.
    \item Else if $b = 1$, sample $\mathbf{x} \sample{\$} X$.
    \item Let $\mathbf{x}^* = \phi_S(\mathbf{x})$ be a portion of $\mathbf{x}$. 
    \item $\adversary$ receives $\mathbf{x}^*$, $c(\mathbf{x})$ and oracle access to $h_{X}$.
    \item $\adversary$ announces $\mathbf{x}' \in \mathbb{D}^m$. If $\mathbf{x}' = \mathbf{x}$ output 1, else output 0.
\end{enumerate}
\end{experiment}

\begin{definition}[Attribute Inference Advantage]
\label{def:ai-adv}
The AI advantage of $\adversary$ on the classifier $h$, i.e., $\text{Adv}_{\text{AI}}(\adversary, h_X, m', n, \mathcal{D})$, is defined as
\begin{align*}
   &\quad \Pr[\text{Exp}_{\text{AI}}(\adversary, h_X, m', n, \mathcal{D}) = 1 \mid b = 1] \\
   &- \Pr[\text{Exp}_{\text{AI}}(\adversary, h_X, m', n, \mathcal{D}) = 1 \mid b = 0].
\end{align*}
\end{definition}

The above definition mirrors the one from~\cite{somesh-overfit}. However, the attribute inference covered in~\cite{somesh-overfit} is more general; it considers arbitrary background knowledge about $\mathbf{x}$, and not necessarily a portion. The version that we consider is called the model inversion attack~\cite{pharma, somesh-overfit}. \blue{We remark that the above definition is by no means the standard definition of AI.} We refer the reader to Section~\ref{sec:rw} for a discussion on other definitions of AI proposed in literature.



\descr{Inferring through the Distribution vs the Model.} Note that these definitions purposely define advantage as the difference between inferring through the distribution alone versus inferring via access to the model. For instance, one way to infer the missing features is to exploit statistical correlations between the observed features and the label. But notice that this can be done directly through knowledge of the distribution, irrespective of access to the model. The AI advantage will therefore be negligible for such a strategy. Hence, the definitions only define an AI attack as advantageous if it can infer more through the model as opposed to through statistical trends of the feature vectors. The same applies to approximate attribute inference to be defined shortly. See Section~\ref{sec:rw} for further discussion on this point. \blue{Correlations can indeed be a privacy issue if the distribution is not known to the attacker. But this definition is outside the scope of our work, where we consider the distribution to be known by the attack algorithm.}

\descr{Relationship between AI and SMI.} It is easy to see how an AI adversary can use an SMI adversary to infer attributes. Given a portion $\mathbf{x}^* = \phi_S(\mathbf{x})$, the AI adversary uses the size of $S$, i.e., $m'$, to choose an $r$ 
\iffull
according to Corollary~\ref{cor:phi}, in Appendix~\ref{app:metrics},
\fi
%
and then runs the SMI adversary with input $r$ and each possible \emph{sibling} of the vector $\mathbf{x}$ (Even though the set $S$ is not explicitly given to the AI adversary, it is implicit from the portion). Whenever, the SMI adversary outputs 1, i.e., predicts the corresponding vector to be a member, our AI adversary outputs that vector as its guess for $\mathbf{x}$. \blue{Thus $\text{SMI} \Rightarrow \text{AI}$.} 



In the other direction, the following theorem shows that AI implies SMI, or in other words $\neg \text{SMI} \Rightarrow \neg \text{AI}$. Therefore, if an SMI adversary has negligible advantage, then we cannot hope to find an AI adversary with significant advantage. 

\begin{theorem}
\label{the:nosmi-noai}
Let $\mathcal{A}$ be an AI adversary with advantage $\delta$. Then there exists an SMI adversary $\mathcal{B}$ with advantage $\delta + \epsilon(r)$, \blue{assuming $\epsilon(r)$, the $r$-neighbor distinguishability advantage, is negligible for small $r$}.
\end{theorem}

\begin{proof}
\iffull
Consider an SMI adversary $\mathcal{B}$ which is given $\mathbf{x}$. SMI chooses a random index, or alternatively, a random index set $S$ of cardinality 1. The adversary $\mathcal{B}$ constructs $\mathbf{x}^* = \phi_S(\mathbf{x})$ and gives it to $\mathcal{A}$. Upon receiving $\mathbf{x}'$ from $\mathcal{A}$, the adversary $\mathcal{B}$ checks if $\mathbf{x}' = \mathbf{x}$. If yes, it returns $1$. Else it returns 0. The advantage of adversary $\mathcal{B}$ is
\begin{align}
    &\Pr [b' = 1 \mid b = 1] - \Pr [b' = 1 \mid b = 0] \nonumber\\
    &= \Pr[\text{Exp}_{\text{AI}}(\adversary, h_X, 1, n, \mathcal{D}) = 1 \mid b = 1] \nonumber\\
   &- \Pr[\text{Exp}^*_{\text{AI}}(\adversary, h_X, 1, n, \mathcal{D}) = 1 \mid b = 0] \label{eq:theo2},
\end{align}
where $\Pr[\text{Exp}^*_{\text{AI}}(\adversary, h_X, 1, n, \mathcal{D}) = 1 \mid b = 0]$ denotes the version of Experiment~\ref{exp:app-attr-inf}, where $\mathbf{x} \leftarrow \mathcal{D}$ in Step 3 is replaced with $\mathbf{x}_0 \sample{\$} X, \mathbf{x} \leftarrow  B_d(\mathbf{x}_0, r)$, according to the distribution induced by $\mathcal{D}$. \blue{From Eq.~\ref{eq:r-neigh} for any algorithm $\mathcal{C}$, we see that:
\begin{align*}
&\Pr[\text{Exp}_{\text{AI}}(\adversary, h_X, 1, n, \mathcal{D}) = 1 \mid b = 0] \\
&-\Pr[\text{Exp}^*_{\text{AI}}(\adversary, h_X, 1, n, \mathcal{D}) = 1 \mid b = 0]\\
&\leq \Pr[ \mathcal{C}(\mathbf{x}, \mathcal{D}) = 1] - \Pr[\mathcal{C}(\mathbf{x}', \mathcal{D}) = 1 ] \leq \epsilon(r),    
\end{align*}
where $\epsilon(r)$ is the $r$-neighbor distinguishability advantage. Thus, Eq.~\ref{eq:theo2} becomes}
\begin{align*} 
   &\Pr [b' = 1 \mid b = 1] - \Pr [b' = 1 \mid b = 0] \\
   &\leq \Pr[\text{Exp}_{\text{AI}}(\adversary, h_X, 1, n, \mathcal{D}) = 1 \mid b = 1] \\
   &- \Pr[\text{Exp}_{\text{AI}}(\adversary, h_X, 1, n, \mathcal{D}) = 1 \mid b = 0] + \epsilon(r)\\
   &= \delta + \epsilon(r).
\end{align*}
Under the indistinguishable neighbor assumption~\ref{def:smooth-dist}, we assume $\epsilon(r)$ to be negligible for small $r$. 
\fi
\end{proof}

\blue{Theorem 2, together with the previous result, shows that $\text{SMI} \Leftrightarrow \text{AI}$, provided the $r$-neighbor distinguishability assumption holds. If $\epsilon(r)$ is large, then the advantage does not translate, as now the neighbor vector (sampled from the induced distribution) does not follow the distribution $\mathcal{D}$ expected by the AI algorithm $\mathcal{A}$ in Experiment~\ref{exp:app-attr-inf}.} This observation is mirrored by our experiments where we show that constructing an attacker that can \emph{exactly} predict the missing values of a portion of a member vector with high probability is highly unlikely.
\blue{Since this equivalence is under the $r$-neighbor distinguishability assumption, SMI is not identical to the notion of AI. This is true in particular for datasets where the assumption fails to hold. For instance, a location dataset with sparse locations. However, the assumption should hold for most real-world datasets, such as the ones considered in this paper.}  
\blue{We remark that in its raw form the definition may be overly strict for continuous attributes. To overcome this, in our experiments we apply binning, and flag any continuous attribute value as correctly identified if it falls in the correct bin (See Section~\ref{sec:ai-result} for the CIFAR dataset). Even with this judicious interpretation of the definition, our experimental results show that the adversary does not have much advantage in predicting the missing attributes.} \blue{This leads to }the definition of approximate AI, that requires the attacker to predict the missing values only ``approximately close'' to a member vector.

\begin{experiment}[Approximate Attribute Inference (AAI)]
\label{exp:app-attr-inf}
Let $\adversary$ be the adversary, let $X \leftarrow \mathcal{D}^n$ be the input dataset, let $S$ be a subset of $[m]$ with cardinality $m'$ such that $1 \le m' < m$, and let $\alpha \ge 0$ be a distance parameter.
\begin{enumerate}
   \item Construct model $h_{X}$.
    \item Sample $b \sample{\$} \{0, 1\}$.
    \item If $b = 0$, sample $\mathbf{x} \sample{} \mathcal{D}$.
    \item Else if $b = 1$, sample $\mathbf{x} \sample{\$} X$.
    \item Let $\mathbf{x}^* = \phi_S(\mathbf{x})$ be a portion of $\mathbf{x}$. 
    \item $\adversary$ receives $\mathbf{x}^*$ and oracle access to $h_{X}$.
    \item $\adversary$ announces $\mathbf{x}' \in \mathbb{D}^m$. If $d(\mathbf{x}',\mathbf{x}) \le \alpha$ output 1, else 0.
\end{enumerate}
\end{experiment}

\begin{definition}[Approx. Attribute Inference Advantage]
\label{def:aai-adv}
The AAI advantage of $\adversary$ on the classifier $h$, i.e., $\text{Adv}_{\text{AI}}(\adversary, h_X, m', n, \alpha, \mathcal{D})$, is defined as

\begin{align*}
   &\quad \Pr[\text{Exp}_{\text{AI}}(\adversary, h_X, m', n, \alpha, \mathcal{D}) = 1 \mid b = 1] \\
   &- \Pr[\text{Exp}_{\text{AI}}(\adversary, h_X, m', n, \alpha, \mathcal{D}) = 1 \mid b = 0].
\end{align*}
\end{definition}

Note that with $\alpha = 0$, Experiment~\ref{exp:attr-inf-somesh} becomes a special case of Experiment~\ref{exp:app-attr-inf}. It is easy to see that AI $\Rightarrow$ AAI, but the converse is not necessarily true. 

\blue{Depending on the distance metric, the AAI advantage definition can have different interpretations. For instance, if the distance metric is Euclidean distance, then this captures the notion of mean squared error. Similarly, the Manhattan distance metric gives the absolute error interpretation. The parameter $\alpha$ should be set carefully to avoid degenerate cases, e.g., if $\alpha$ is set too small, then an adversary whose guess is always slightly off $\alpha$ would be deemed less advantageous than an adversary with only one guess within $\alpha$ and the remaining deviating significantly from $\alpha$. For our experiments, we set $\alpha$ as the distance of a random guess from the target vector.}




\descr{Computing Advantages in Practice.} As most prior work on membership inference  uses the Area Under the Curve (AUC) of a Receiver Operating Characteristics (ROC) curve as a measure of aggregated classification performance of the MI attacker (viewed as a binary classifier), we use the same metric in our experiments in Section~\ref{sec:experiments}. 
\iffull
In Appendix~\ref{app:misc},
\fi
we show how our advantage definitions~\ref{def:mi-adv} and \ref{def:smi-adv} are related to the AUC statistic. 
For the evaluation of AI and AAI attacks we employ the advantage metrics defined in Definitions~\ref{def:ai-adv} and~\ref{def:aai-adv}.

\section{Experimental Methodology}
\label{sec:experiments}

In this section, we describe the datasets, instances of MI and AI attacks used, and how we carry out membership and attribute inference attacks in our experiments in Sections~\ref{sec:mia-exp} and \ref{sec:aia-exp}. We first evaluate the performance of several MI attacks in terms of MI advantage (Def.~\ref{def:mi-adv}) with increasing distance of the challenge vectors from the training set (Section~\ref{sec:mia-exp}). We then evaluate the performance of AI attacks in terms of AI advantage (Def.~\ref{def:ai-adv}) which use MI attacks as a subroutine (Section~\ref{sec:ai-result}). Finally, we study the performance of the same AI attacks in the sense of approximate attribute inference (Def.~\ref{def:aai-adv}). 
These experiments demonstrate the shortcomings of MI and AI definitions and the need for our newly proposed definitions, i.e., SMI and AAI.




\subsection{Data and Machine Learning Models}

\label{sec:datasets}
We evaluate MI and AI attacks on three different datasets: (a) \emph{Location:} a social network locations check-in dataset obtained from Foursquare~\cite{locationdataset}, (b) \emph{Purchase:} a shopping transactions dataset~\cite{purchase_dataset}, and (c) \emph{CIFAR} an image dataset~\cite{krizhevsky2009learning}. These datasets have previously been used to demonstrate MI~\cite{ml-leaks, shokri-mia, jayaraman2019evaluating} and AI attacks~\cite{jayaraman2019evaluating}. The first two datasets are binary, with 467 binary features in Location and 599 in Purchase, whereas the CIFAR dataset was processed, using principal component analysis (PCA), to yield 50 continuous features normalized between $-1$ and $1$~\cite{jayaraman2019evaluating}. We applied k-means clustering to obtain class labels in both the Location and Purchase datasets. The number of classes in the Location dataset is 30 and for the Purchase dataset, we create 5 variants differing in the number of classes (2, 10, 20, 50, 100), as is done in \cite{ml-leaks}. Finally, the CIFAR dataset contains 100 class labels for the images, with an additional set of 20 labels which are a superset of the 100 classes, e.g. the label ``flowers'' is the superset of orchids, poppies, roses, sunflowers, and tulips. We call the two datasets CIFAR-100 and CIFAR-20.


We predominantly explore the neural network as our target model. However, later in Section~\ref{sec:mia-models}, we show that our observations generalize to Logistic Regression, Support Vector Machine, and Random Forest classifiers. The exact configurations of these models for each experiment are detailed in Appendix~\ref{sec:model_params}.

\subsection{MI and AI Adversaries}
\label{sec:attack-how}

We use five MI attacks from literature as examples of an MI adversary (Def.~\ref{def:mi-adv}), and three AI attacks as examples of an AI adversary (Def.~\ref{def:ai-adv}). 

\subsubsection{MI Attacks}
\label{sec:mia-how}

Our MI attacks include three black-box attacks: the shadow model based attack from Shokri et al.~\cite{shokri-mia}, the attack from Yeom et al. based on prediction loss~\cite{somesh-overfit}, and the attack from Salem et al. based on maximum prediction confidence~\cite{ml-leaks}, and two variants (local and global) of a white-box  attack from Nasr et al.~\cite{nasr2018comprehensive}. Recall that in an MI attack, the attacker is given a member or a non-member vector with optionally its true label, and is asked to infer membership.

\descr{Shadow MI~\cite{shokri-mia}.}
This attack trains a machine learning model, called an attack model, to discern membership of a given vector from the prediction output vector (confidence of every class label). This attack model leverages outputs from shadow models which are trained with a disjoint dataset to mirror the behaviour of the target model.

\descr{Loss MI~\cite{somesh-overfit}.}
This attack eliminates the high computational cost of training shadow and attack models by evaluating the prediction loss of a vector on the target model directly. This attack, in practice, may use the target model training loss as a loss threshold to determine membership.

\descr{Conf MI~\cite{ml-leaks}.}
Conf MI, short for Confidence, is even simpler than Loss MI; instead of computing the prediction loss, the attack simple uses the confidence value of the most likely label. With less information available to the attack, it performs worse than both Loss MI and Shadow MI (as we shall see in Section~\ref{sec:mia-exp}). However, it is arguably a more practical attack, requiring less information.

\descr{Local White Box (WB) and Global White Box (WB) MI~\cite{nasr2018comprehensive}.}
The three previous attacks are all black-box attacks with little to no information about the target model, and only API access to the model. An alternative form of MI attack is a white-box membership inference attack, which in a federated setting, may offer additional information for an adversary to launch an MI attack. Despite the federated setting, we suspect any observations we perform on the black-box setting should be reflected in a white-box setting. Nasr et al. attack~\cite{nasr2018comprehensive} is a standalone attack targeting federated machine learning models in a white-box setting. 
The white-box setting lends additional hidden layer information and intermediate model states from the training process to better inform the attack model. This information includes the final layer gradients, outputs and the true label, obtained from intermediate and final states of the target model. 

The federated setting consists of multiple parties, each training models independently and contributing parameters to a central server. 
The server aggregates these parameters before sending the results back to each party to replace their individual model. 
Two different attacks are tested: the \textit{Global WB MI} attack, where the attacker has server level information and attacks each of the parties individually (in the case of a Malicious MLaaS provider); and the \textit{Local WB MI} attack whereby the attacker is an external or contributing party attacking the server or MLaaS provider.


\subsubsection{Attribute Inference (AI) Attacks}
\label{sec:aia_how}

We use three AI attacks as examples of an AI adversary. All three attacks use an MI attack as a subroutine as mentioned in Section~\ref{sec:formal}. We, therefore, use the same names for them as the underlying MI attacks. Briefly, our general procedure to evaluate an AI attack is as follows. Given a portion $\mathbf{x}^* = \phi_S(\mathbf{x})$ for a set $S$ of unknown features (cf. Def.~\ref{def:portions}), we first construct all siblings of $\mathbf{x}$ (cf. Def.~\ref{def:siblings}), by trying all possible permutations of the missing attribute(s), i.e., features. We then give each sibling as input to the MI attack. From the set of siblings, the vector with the highest membership confidence from the underlying MI attack is deemed the original vector $\mathbf{x}$, and thus its attributes identified as the missing attributes.

\descr{Shadow AI.}
The basis of this attack is to use the attack model from Shadow MI~\cite{shokri-mia} for AI. While the MI version of the attack only uses the final decision (member or non-member), in the AI attack, we use the prediction confidence from the attack model to gauge which vector is most likely the original vector, and thus infer attributes.

\descr{Loss AI~\cite{somesh-overfit}.}
This attack follows the original proposal from Yeom et al. to use the training loss as the deciding factor for attribute inference. Given all siblings, the vector that achieves the prediction loss (from the target model) closest to the training loss, is flagged as the original vector.

\descr{Conf AI~\cite{zhao2019inferring}.}
Recall that Conf MI~\cite{ml-leaks} uses the single largest prediction confidence of the vector to deduce its membership. We repeat the same process, and flag the highest confidence vector (prediction confidence from the target model) from all siblings as the original vector.

\descr{Note.} Although both Local WB and Global WB MI attacks can also be used to perform AI, we opted against, as they are computationally more demanding than other attacks. Fortunately, as we shall show, Local WB and Global WB MI attacks show similar trends as the other 3 MI attacks we use as subroutines for AI.

\subsection{Attack Methodology}
\label{sec:infer-method}

Prior to inference, we must first train a target model on a given dataset. To do so we split the dataset into training and testing sets. We describe the exact training/testing data split, the architecture of the neural network, and other hyper-parameters in Appendix~\ref{sec:model_params}. These models have been tuned to replicate models observed in prior works. The training set is used to train the target model, and the prediction accuracy of the target model is evaluated on the testing set. We tune our target models to produce prediction accuracies comparable to \cite{shokri-mia} (exact attack accuracy values are reported in Table~\ref{tab:model_traintest} in Appendix~\ref{sec:model_params}). From the training and testing sets we then sample 1000 vectors each to serve as our member and non-member sets.
With the target model prepared, we take the following steps to launch MI and AI attacks.

\descr{MI.} For MI, we obtain AUCs by evaluating
the member and non-member subsets with either the MI attack model (for Shadow, Local WB and Global WB MI), or the target model (for Loss and Conf MI) for a membership confidence score.


\descr{AI.} For AI, we take our set of member and non-members, and then use the top most informative features according to the Minimal Redundancy Maximal Relevance (mRMR) criterion~\cite{peng-mrmr}. 
Intuitively, the informative features are likely to have more influence on the classifier's output. This also follows previous work~\cite{wu-influence, somesh-overfit} where it is shown that informative features, i.e., those with more influence, have a positive impact on attribute inference, albeit the results apply for Boolean and binary variables. Thus, the use of most informative features increases the likelihood of an AI attack. The set of most informative features forms the set $S$ of unknown features. For each vector, we then create its portion based on $S$, and generate all siblings of the vector, only one of which is the original vector with the target attribute values.  
With this set of siblings, for each member and non-member vector, we perform an MI attack. This produces a measure of membership confidence (either as attack model probability, prediction loss, or prediction confidence, c.f. Section~\ref{sec:aia_how}). From this measure, the sibling with the highest  membership confidence is regarded as the correct vector, and consequently containing the correct missing attributes. For AI, we regard the attack as a success when the recovered sibling is exactly equal to the original vector (Exp.~\ref{exp:attr-inf-somesh}). For AAI, we regard the attack a success when the recovered sibling is within a given $\alpha$ distance away from the correct attributes (Exp.~\ref{exp:app-attr-inf}).



\section{Membership Inference}
\label{sec:mia-exp}
We first show results from MI attacks highlighting the need for our definition of strong membership inference (SMI) (Exp.~\ref{exp:strong-mem-inf}). Two key findings are:
\begin{itemize}[noitemsep, topsep=0pt, leftmargin=*]
    \item MI attacks perform better if the non-members are at a greater distance from the training dataset. This  observation is crucial for attribute inference, as we shall see in the next section.
    \item MI attack performance is not uniform across all classes in the dataset. In fact, it is inversely related to the dominance of the class, i.e., the decision region of the class (Def.~\ref{eq:dr}).
\end{itemize}





\subsection{MI Attacks on Neural Networks}
\label{sec:mia-5adversaries}

We first inspect the performance of the five MI attacks (See Section~\ref{sec:mia-how}) on members and non-member vectors from the original dataset as a function of their distance from the training dataset (Def.~\ref{def:dist}). We observe that the vectors in the original dataset are quite far away from each other, consequently lacking MI performance information at small distances. Thus we follow this analysis with MI performance on synthetically generated vectors, to illustrate a complete picture of MI performance as a function of distance from the training dataset (Section~\ref{sec:mia-synthetic}). We also explore the relationship between MI attack performance and the decision region of a class (Section~\ref{sec:mia-labels}). 



\subsubsection{MI Performance on the Original Dataset as a Function of Distance}
\label{sec:mia-original}

   \begin{figure*}[!ht]
     \centering
     \subfloat[Conf MI\label{fig:nn_salem_nmvec_auc}]{%
       \includegraphics[width=0.32\textwidth]{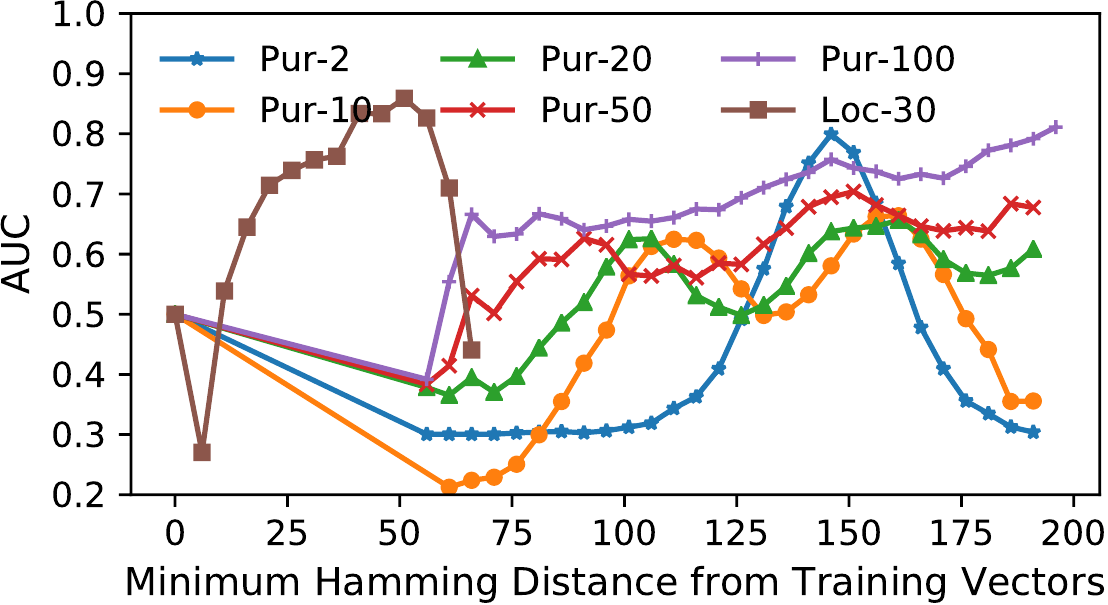}
     } \hfill
     \subfloat[Loss MI\label{fig:nn_yeom_nmvec_auc}]{%
       \includegraphics[width=0.32\textwidth]{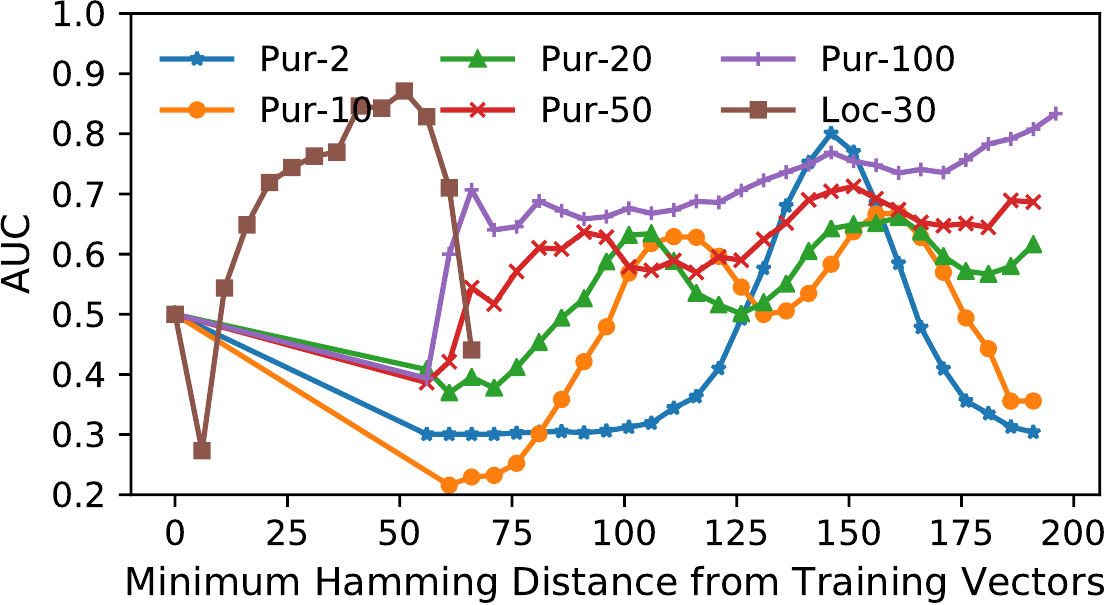}
     } \hfill
     \subfloat[Shadow MI\label{fig:nn_shokri_nmvec_auc}]{%
       \includegraphics[width=0.32\textwidth]{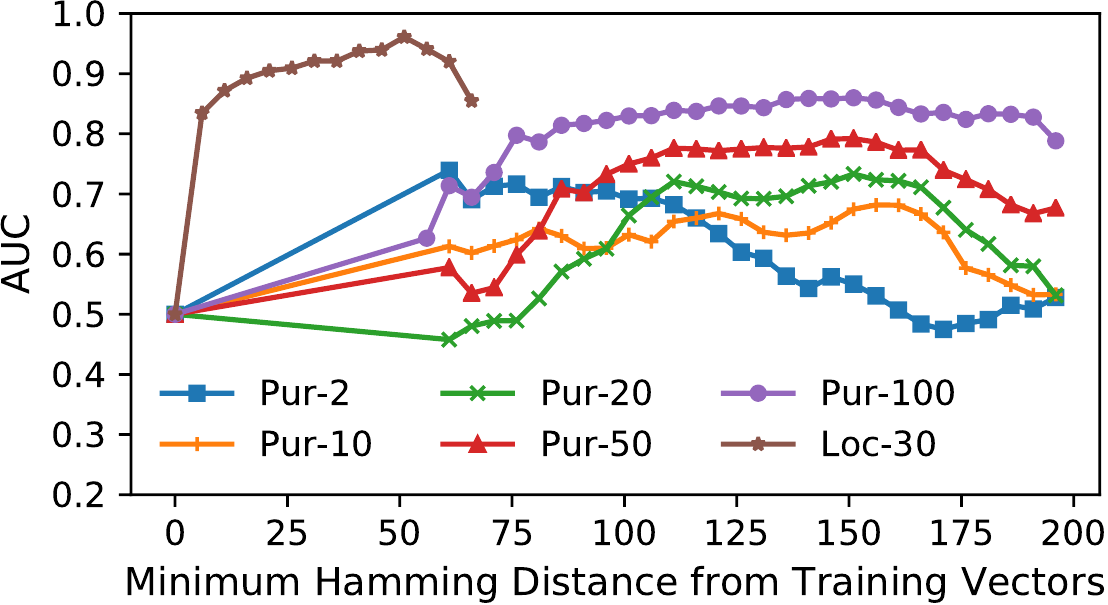}
     }
     \\ \vspace{-3mm}
     \subfloat[Local WB MI\label{fig:nn_local_nasr_nmvec_auc}]{%
       \includegraphics[width=0.32\textwidth]{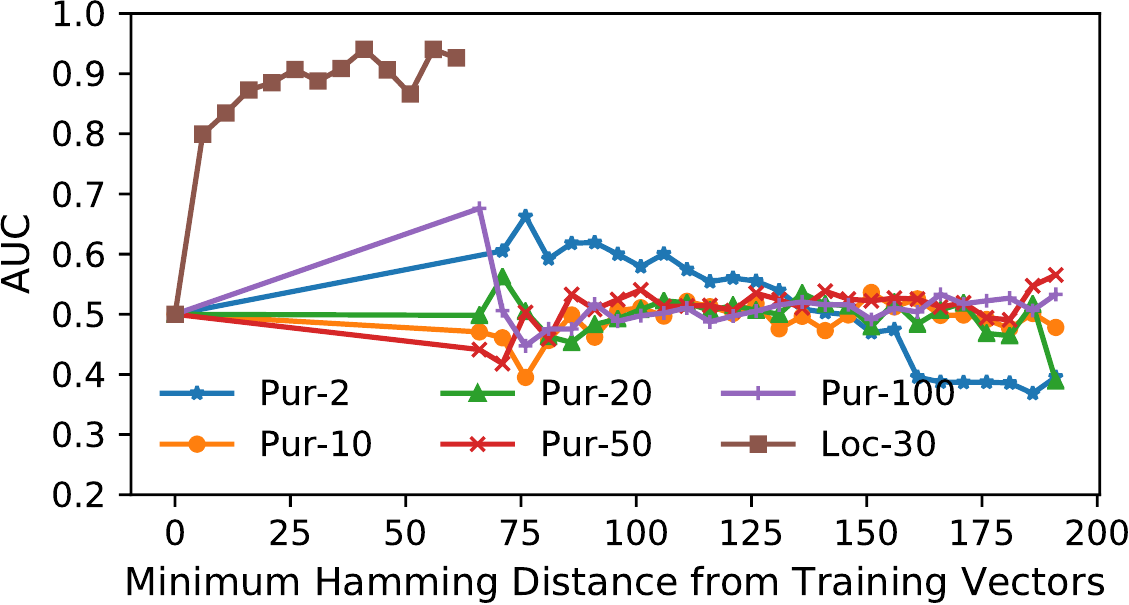}
     } \hfill
     \subfloat[Global WB MI\label{fig:nn_global_nasr_nmvec_auc}]{%
       \includegraphics[width=0.32\textwidth]{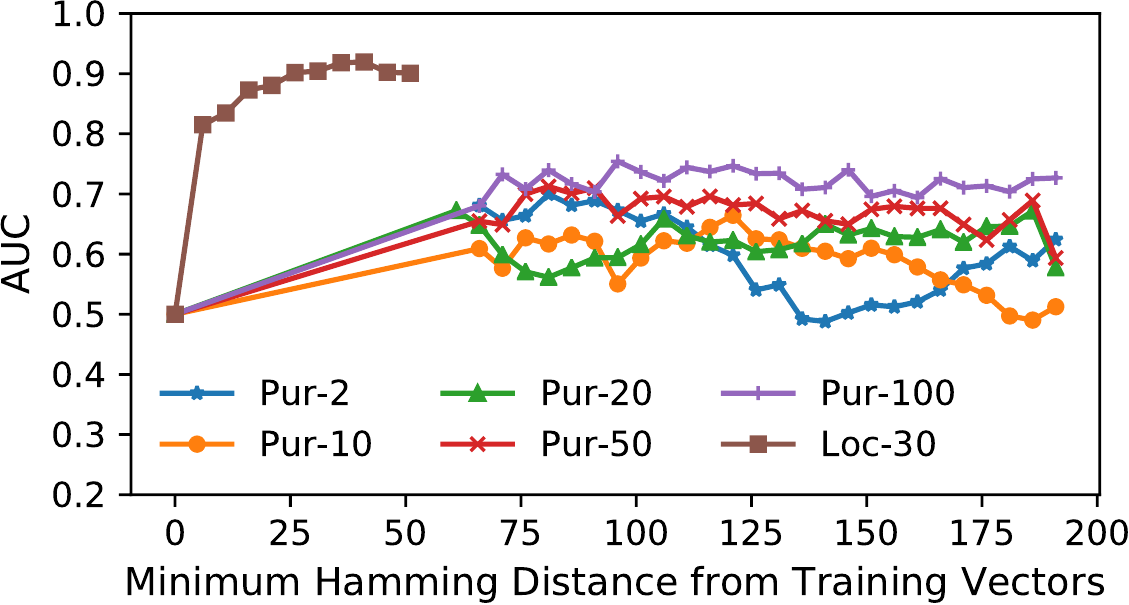}
     } \hfill
     \subfloat[CIFAR-100\label{fig:nn_cif100_nmvec_auc}]{%
       \includegraphics[width=0.31\textwidth]{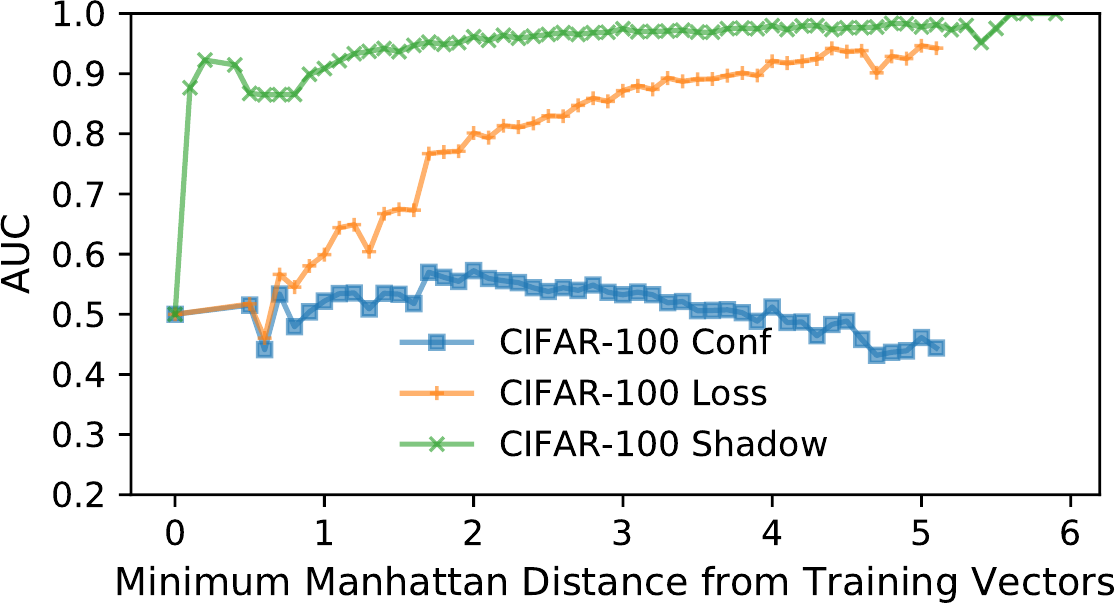}
     }
     \caption{Increasing AUC of various MI attacks with increasing Hamming distance of original non-members from the training dataset on target models. Subplot (f) compares the difference in attack AUC between MI attacks on CIFAR-100 (CIFAR-20 can be found in Appendix~\ref{sec:appendix-cifar20}).}
     \label{fig:mia_nmvec_auc}
     \vspace{-2mm}
   \end{figure*}

After training the target model, we compute the distance of each non-member vector from the training set. Recall from Section~\ref{sec:formal}, we use Hamming distance $d_H$ for Location and Purchase datasets (which are binary), and Manhattan distance $d_M$ for the continuous (normalized) CIFAR datasets. The vectors are then grouped according to their distance from the training dataset (the distance is 0 for members). We then calculate AUC for each distance by taking the membership score of each vector in this distance group as the negative class, and all member vectors as the positive class. This test is repeated 50 times (10 for the WB MI attacks due to computational resource limitations), and the AUC is computed on the aggregation of all confidence values (Fig.~\ref{fig:mia_nmvec_auc}).

\descr{Results.} From Figs.~\ref{fig:nn_salem_nmvec_auc} to \ref{fig:nn_global_nasr_nmvec_auc}, we observe that for the Location dataset the AUC improves as the distance of non-members from the training dataset increases in all five MI attacks, with the AUC being closer to random guess (0.5) for non-members closest to the training dataset. From the same figures, we can see that this trend is less obvious for the Purchase datasets. This is mainly because non-members in the Purchase datasets are at a greater distance from the training dataset. The same observation can be made for CIFAR-100 in Fig.~\ref{fig:nn_cif100_nmvec_auc} (results for CIFAR-20 are in Appendix~\ref{sec:appendix-cifar20}).
This gives a first indication that SMI (Exp.~\ref{exp:strong-mem-inf}) is less successful than MI (Exp.~\ref{exp:mem-inf-somesh}).



An issue with the results in Figure~\ref{fig:mia_nmvec_auc} is that there is a lack of vectors close to and farthest away from the training datasets. This is evident from the distribution of distances displayed in Fig.~\ref{fig:nmvec_hist_both}. 
Observe that there is little data available when we attempt to inspect AUC for distances close to the original dataset. As the non-members in the original Purchase datasets do not provide a full picture of how the MI performance behaves across all distances, and hence MI performance, in the next section, we generate synthetic vectors allowing us to control the distance (Hamming or Manhattan) from the training dataset providing a more complete picture.

A few other observations are worth highlighting:

\begin{itemize}[noitemsep, topsep=0pt, leftmargin=*]
    \item Consistent with what has been previously reported on MI attacks, the attack accuracy improves on target models with a greater number of classes~\cite{ml-leaks, shokri-mia}. Higher number of classes is also linked to a higher degree of overfitness (Table~\ref{tab:model_traintest}). 
    \item The AUC performance of the Loss and Conf MI attacks is almost identical. Recall that Conf MI uses the maximum confidence value of the prediction, while Loss MI uses the prediction loss. Note that the prediction loss for a classification model is simply the loss between the confidence of the true label and $1$. 
    Given that a (good) target model is likely to predict the correct label of the vector, it follows that, most of the times, the maximum prediction confidence (as used in Conf MI) will be equal to the confidence used to compute the loss in Loss MI.
    \item Some of the AUCs exhibit peaks; an increase as the distance from the training dataset increases followed by a decrease. This is due to the decision regions (DR) learnt by the classifiers. We shall elaborate on this in Sections~\ref{sec:mia-synthetic} and \ref{sec:mia-labels}.  
    \item Another peculiar observation is that some of the AUCs drop below 0.5, meaning that the strategy employed by the corresponding MI attack predicts flips and applies more to non-members than to members. The potential reason behind this is the same as the observation above which we shall explain in Section~\ref{sec:mia-labels}.
\end{itemize}



   \begin{figure}[t]
   \centering
     \subfloat[Hamming distance\label{fig:nmvec_hist}]{%
       \includegraphics[width=0.48\columnwidth]{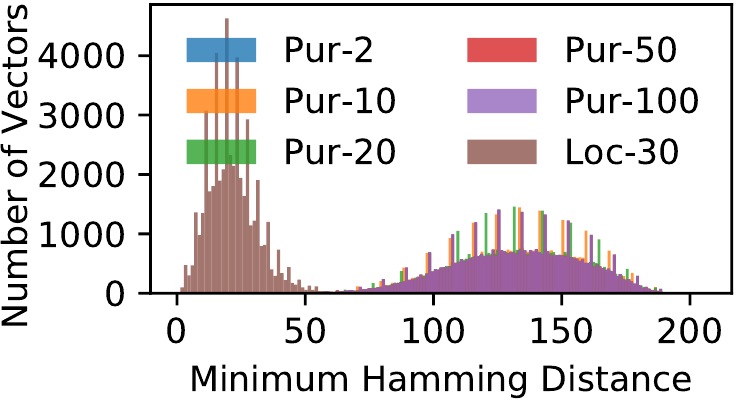}
     }
     \subfloat[Manhattan distance\label{fig:nmvec_hist_cifar}]{%
       \includegraphics[width=0.48\columnwidth]{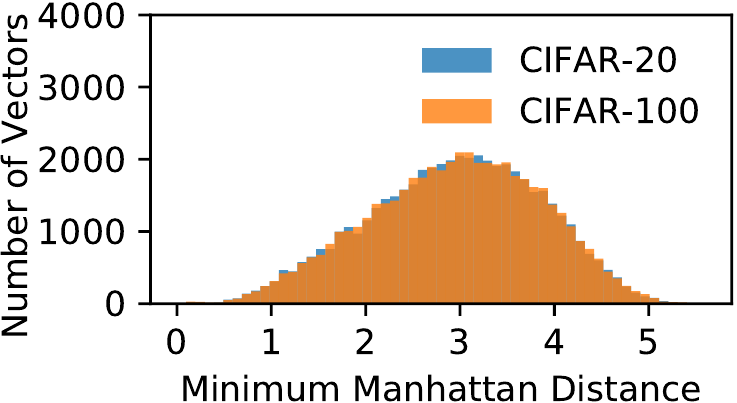}
     }
     \caption{Histogram of distances of non-members from members in our training datasets. This data distribution is consistent across all attacks.}
     \label{fig:nmvec_hist_both}
    \vspace{-3mm}
   \end{figure}

\begin{observation}
\label{obs:mi-dist}
In the MI attacks reported in literature, the distance of non-members from the training dataset is large. In general, an MI attack is more likely to accurately predict a non-member, the greater its distance from the training dataset. 
\end{observation}




\subsubsection{MI Performance on Synthetic Non-Members as a Function of Distance}
\label{sec:mia-synthetic}

Ideally, synthetic vectors should follow the original data distribution. Unfortunately, this would not yield vectors close to the training dataset as can be seen from Figure~\ref{fig:nmvec_hist_both}. To circumvent this, we take existing vectors and create synthetic vectors by flipping or perturbing some of the features. This creates synthetic vectors \blue{that are deliberately off-manifold, but still close to a training vector, where the majority of unaltered features still follow the original data distribution}, while allowing us to control distance from the training dataset.

To generate synthetic vectors for the binary datasets (Location and Purchase), we (a) randomly select a member of the training set, (b) randomly select features to invert, (c) and vary the number of features and generate $5$ non-members for each distance group, ranging from Hamming distance $1$ to, $467$ for Location, and $599$ for Purchase.
For CIFAR datasets, we define Manhattan distance groups at increments of $0.05$ from the training dataset, starting from $0.05$ to $5$. We then produce non-members by randomly selecting features and adding additive perturbations to the feature values of the original vector. The process is repeated $5$ times for each Manhattan distance group. The entire process is repeated for all selected 1000 member vectors for each dataset. The distance to the training dataset is recomputed for all non-members, to cater for the event that the nearest neighbor of a non-member in the training dataset has changed. 
The vectors thus generated are non-members, with the same label as the original member, unless, by chance, any of them collides with a member, in which case we discard it. We also ensure that the nearest neighbor in the dataset of the newly generated vector is of the same label as the base member vector, if not, this generated vector is discarded.


\descr{Results.} The AUCs of the five MI attacks are displayed in Fig.~\ref{fig:mia_genvec_auc}. For all five attacks, we observe that the AUC is close to 0.5 for vectors close to the training dataset, and starts improving as the distance from training dataset increases. It is also evident that the higher the number of classes, the steeper the improvement in AUC as the Hamming distance increases for the Location and Purchase datasets. 
This is more obvious through the magnified Fig.~\ref{fig:mia_genvec_auc_zoom}, where we show AUC of the Conf MI attack on the Location, Purchase and CIFAR datasets at smaller distances from the datasets. The AUC is below 0.6 for Hamming distances of less than $5$ and Manhattan distance of less than $0.2$. This implies that the MI attack is not successful enough in the stronger sense, i.e., in the sense of SMI (Def.~\ref{def:smi-adv}). This has implications for attribute inference, as we shall see in Section~\ref{sec:aia-exp}.

On datasets with higher number of classes, the AUCs of Loss MI (Fig.~\ref{fig:nn_yeom_genvec_auc}), Local WB (Fig.~\ref{fig:nasrloc_genvec_auc}) and Global WB (Fig.~\ref{fig:nasrglob_genvec_auc}) MI, show little change after a certain distance, even if the distance of non-members from the training dataset increases. On the other hand, on the Purchase datasets, for smaller number of classes (2, 10 and 20), Conf (Fig.~\ref{fig:nn_salem_genvec_auc}), Loss (Fig.~\ref{fig:nn_yeom_genvec_auc}) and Shadow (Fig.~\ref{fig:nn_shokri_genvec_auc}) MI attacks observe an increase in AUC followed by a decrease. For the 10 and 20 class variants, we see a second incline in the AUC performance of Shadow MI around a Hamming distance of 250. The reason for this is that at certain distances a non-member vector $\mathbf{x}'$ with a class label $j$, might be in the decision region of another class, even when the nearest neighbor of $\mathbf{x}'$ in the dataset has the class label $j$. We elaborate this in the next section. {Interestingly, in Fig.~\ref{fig:nn_cif100_genvec_auc}, the AUC curves of Conf and Loss MI diverge as the Manhattan distance from the training dataset grows greater than 0.7-0.8. This is because at larger Manhattan distance, the target model starts giving incorrect label predictions. The Loss MI attack detects this (as it computes loss with the predicted confidence). On the other hand, Conf MI only uses the highest confidence. It is therefore unable to detect this, showing worse performance. Finally, we note that a few of the AUC lines are ragged, especially at distances furthest away from the datasets. This is exhibited by attack model based MI attacks (Shadow, Local and Global WB). This is because the underlying attack models have less exposure to vectors at large distances as a result of the data distribution (c.f. Fig.~\ref{fig:nmvec_hist}, corresponding to distances where the AUC lines becomes ragged). The AUC curves of Loss and Conf MI are smooth throughout.}



   \begin{figure*}[!ht]
   \centering
     \subfloat[Conf MI\label{fig:nn_salem_genvec_auc}]{%
       \includegraphics[width=0.32\textwidth]{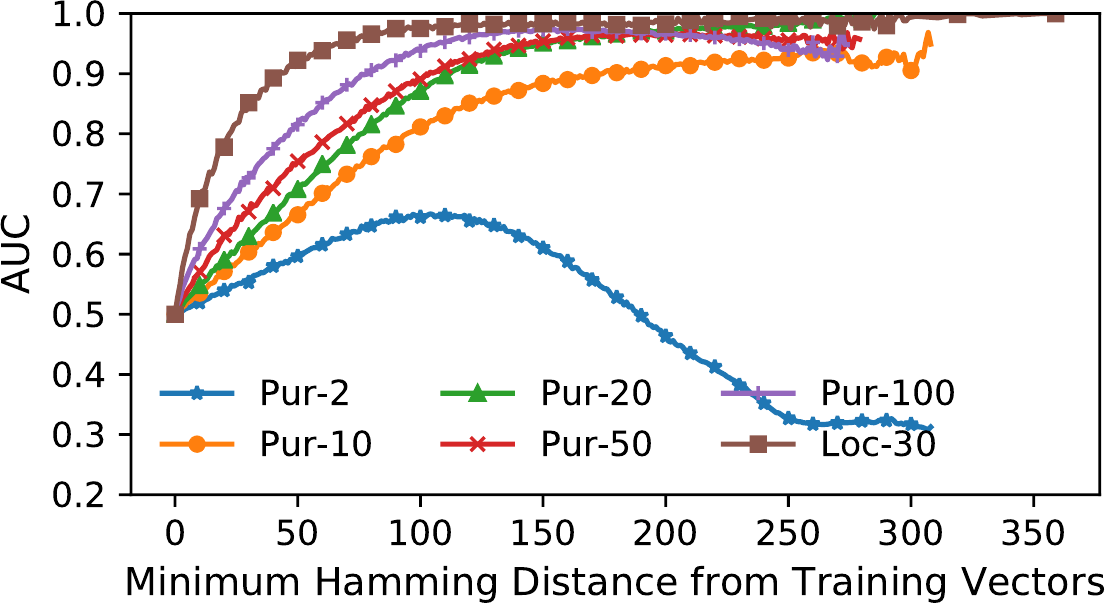}
     } \hfill
     \subfloat[Loss MI\label{fig:nn_yeom_genvec_auc}]{%
       \includegraphics[width=0.32\textwidth]{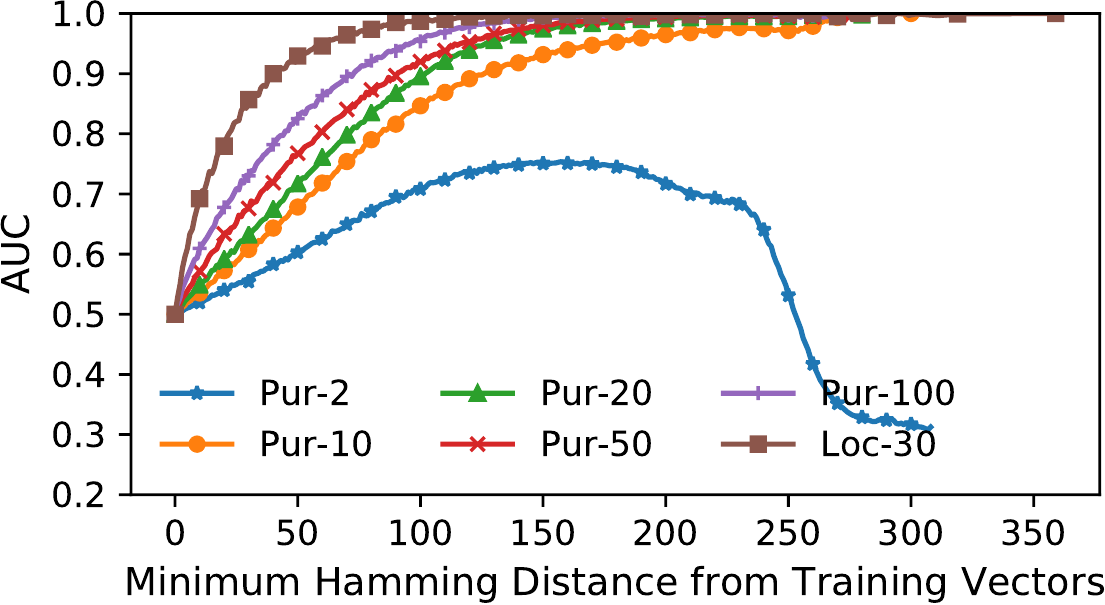}
     } \hfill
     \subfloat[Shadow MI\label{fig:nn_shokri_genvec_auc}]{%
       \includegraphics[width=0.32\textwidth]{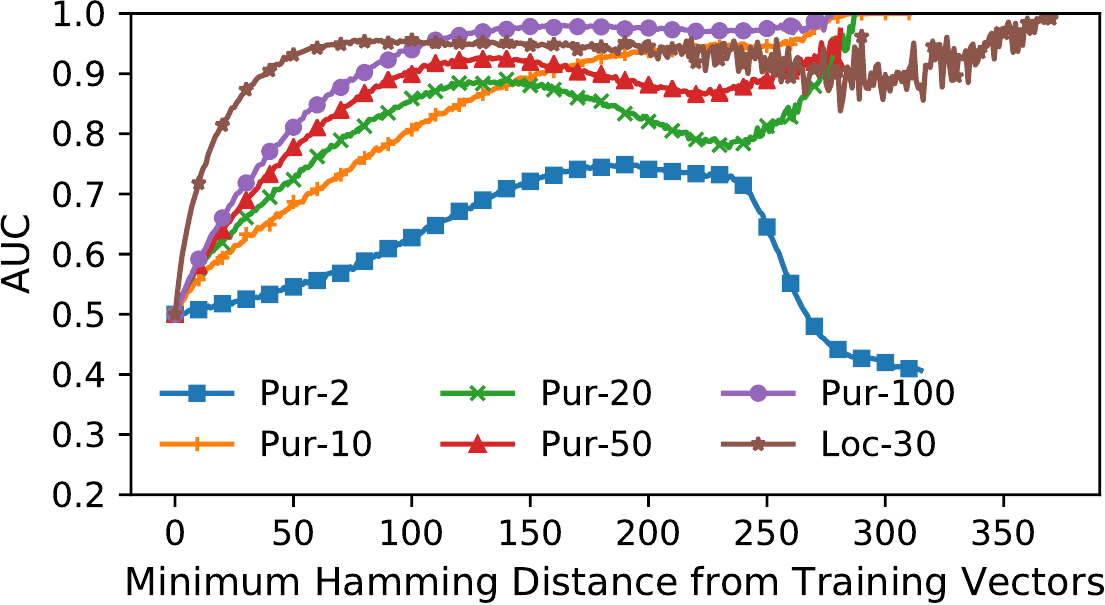}
     }\\ \vspace{-3mm}
     \subfloat[Local WB MI\label{fig:nasrloc_genvec_auc}]{%
       \includegraphics[width=0.32\textwidth]{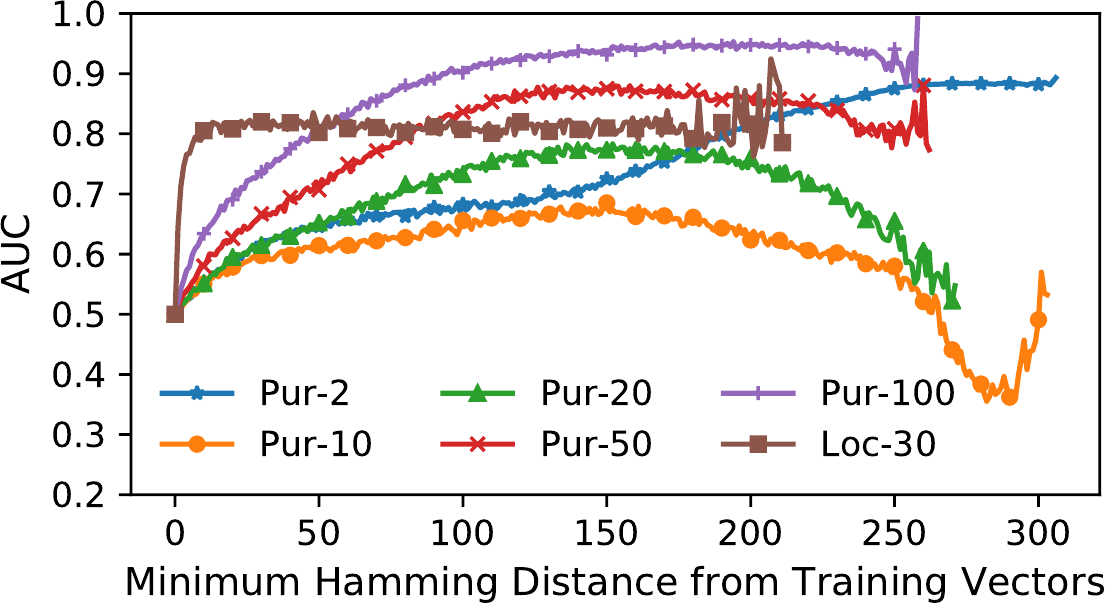}
     } \hfill
     \subfloat[Global WB MI\label{fig:nasrglob_genvec_auc}]{%
       \includegraphics[width=0.32\textwidth]{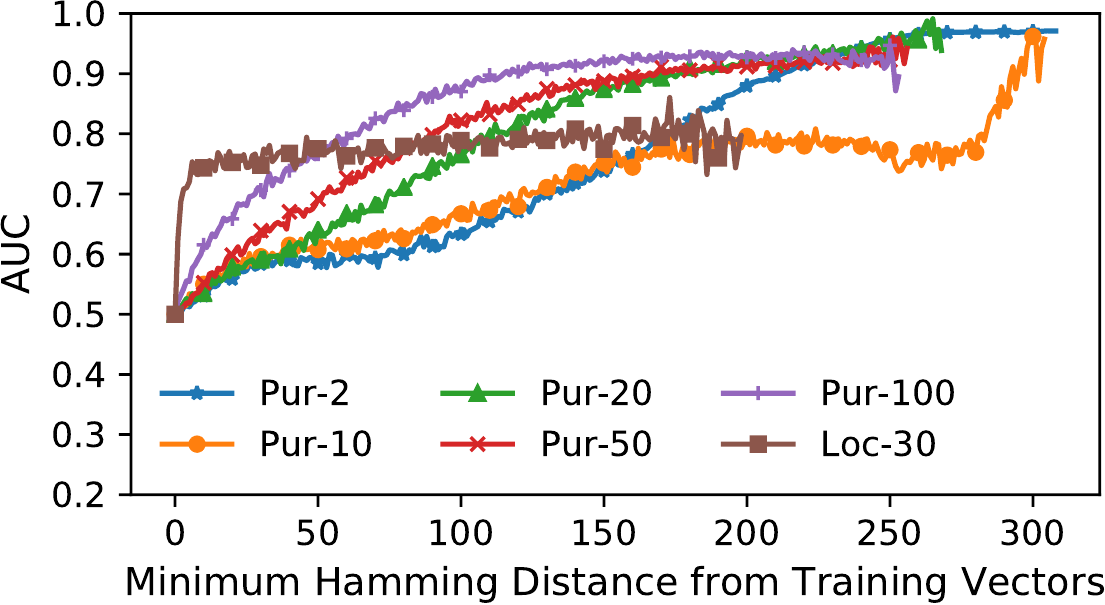}
     } \hfill
     \subfloat[CIFAR-100\label{fig:nn_cif100_genvec_auc}]{%
       \includegraphics[width=0.31\textwidth]{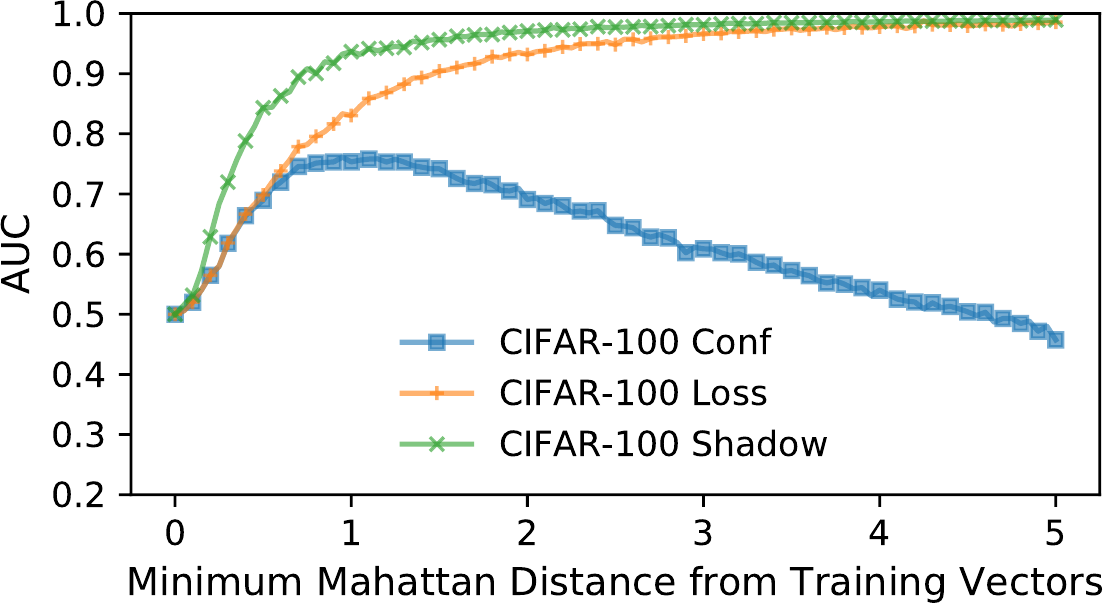}
     }
     \caption{Increasing AUC of various MI attacks with increasing Hamming distance of synthetic non-members from the training dataset on target models. (f) compares the difference in attack AUC between MI attacks on CIFAR-100 (CIFAR-20 can be found in Appendix~\ref{sec:appendix-cifar20}).}
     \label{fig:mia_genvec_auc}
   \end{figure*}

\begin{observation}
\label{obs:smi}
The existing success of MI is a consequence of most non-member vectors being very different to members in terms of distance. For non-member vectors very close to members, the MI attacks perform similar to a random guess (0.5 AUC), and hence fail in the sense of SMI. Thus, the incumbent definition of MI does not capture the behavior of an MI adversary for non-members at distances close to the training data, i.e., SMI, which is essential for launching attribute inference attacks (Theorem~\ref{the:nosmi-noai}). 
\end{observation}


\subsubsection{MI performance on Synthetic Non-Members as a Function of Class Label and Distance}
\label{sec:mia-labels}

   \begin{figure*}[t]
   \centering
     \subfloat[Conf MI\label{fig:nn_salem_label_auc}]{%
       \includegraphics[width=0.195\textwidth]{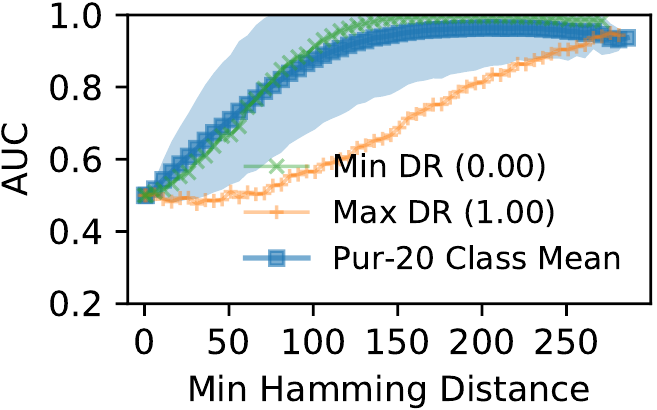}
     } 
     \subfloat[Loss MI\label{fig:nn_yeom_label_auc}]{%
       \includegraphics[width=0.195\textwidth]{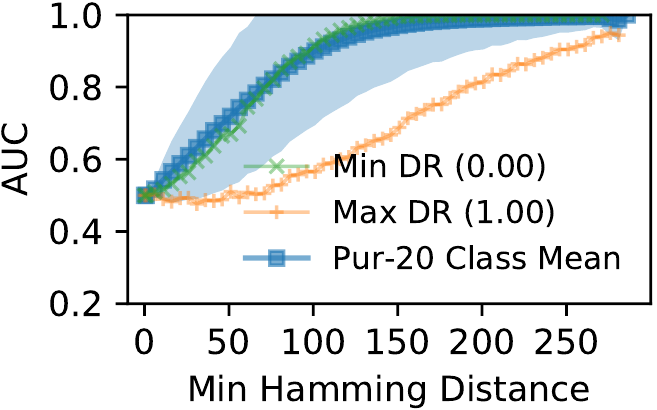}
     } 
     \subfloat[Shadow MI\label{fig:nn_shokri_label_auc}]{%
       \includegraphics[width=0.195\textwidth]{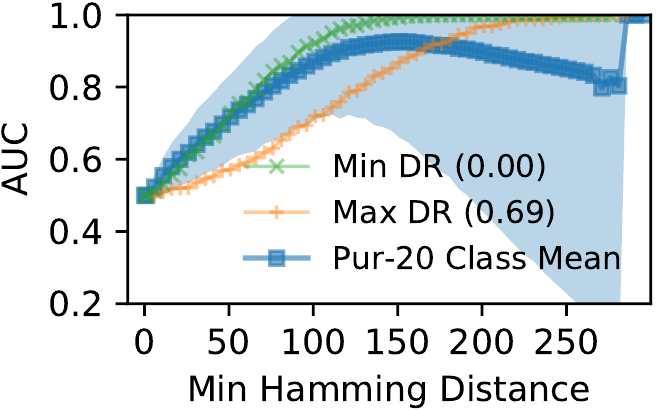}
     } 
     \subfloat[Local WB MI\label{fig:nasrloc_label_auc}]{%
       \includegraphics[width=0.195\textwidth]{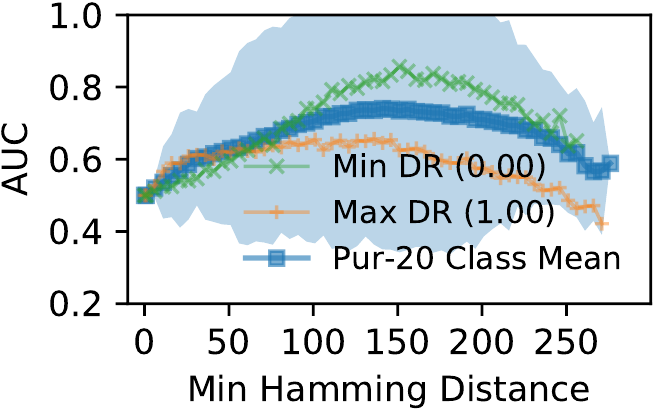}
     } 
     \subfloat[Global WB MI\label{fig:nasrglob_label_auc}]{%
       \includegraphics[width=0.195\textwidth]{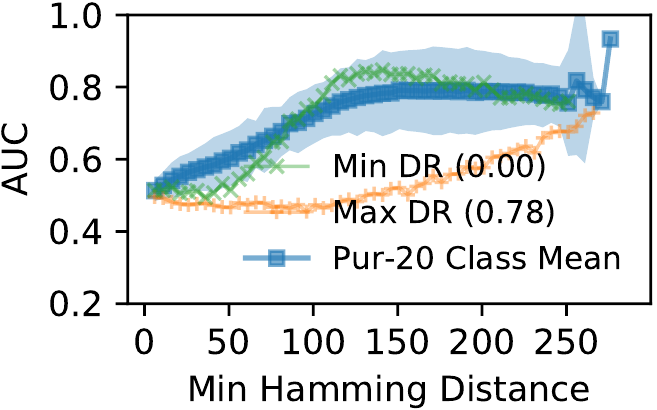}
     }
     \caption{Increasing AUC of various MI adversaries with increasing Hamming distance of synthetic non-members from the training dataset on target models, with a separation of class labels depending on the size of the Decision Region (DR), for the \textbf{Purchase-20} dataset.}
     \label{fig:mia_label_auc}
    \vspace{-3mm}
   \end{figure*}
   

The results thus far have been averaged over members and non-members from all classes. However, as we shall show, the performance of the MI attacks is not consistent over all classes. In fact, the more dominant a class, i.e., the larger the decision region (DR) of the class (Def.~\ref{def:dr}), the less likely it is to be susceptible to membership inference. We empirically measure the decision region of a given class by sampling one million vectors from the feature space by sampling each feature uniformly at random within feature bounds (see feature bounds in Section~\ref{sec:datasets}). A similar approach has been adopted in \cite{zhao2020resilience} for binary classification.

For per-class analysis, we train the target model and generate the synthetic vectors as before, except that now not only do we group synthetic vectors by the distance from the training dataset, but also according to the class label of the nearest training dataset vector. Due to space restrictions, we only show results for the Purchase-20 dataset. Results from the other datasets are in agreement with the conclusions drawn here, and are presented in Appendix~\ref{sec:appendix-perlabel}. 
In the figures, we highlight the AUC performance of the most dominant (largest DR) and least dominant (smallest DR) classes. 

\descr{Results.} Each plot in Fig.~\ref{fig:mia_label_auc} has 4 salient features. A blue line representing the mean AUC of all classes, an accompanying blue shaded area representing 2 standard deviations of AUC between classes, a green and blue line representing the class with the smallest DR, and the largest DR, respectively. From Fig.~\ref{fig:mia_label_auc}, we observe that across all MI attacks, the AUC of the most dominant class is well below the average. In particular, at distances close to the dataset. 

This can be explained as follows. Near the dataset, a non-member vector with class label $j$ (which is also the label of its nearest neighbor in the dataset) is likely to lie in the decision region $\mathcal{R}_j$ of class $j$. As we move away from the dataset, by varying the distance, the corresponding non-member vectors shift further away from the spot in the decision region occupied by their nearest neighbors in the dataset. At certain distance, depending on the target or attack model, the decision region changes to a decision region occupied by a different class, even though the nearest neighbor still has the class label $j$. These non-members are then likely to be misclassified as member vectors of another class, since they lie deep in the decision region of another class. This phenomenon is particularly true if one class overwhelmingly dominates other classes, thus occupying the bulk of the decision region. In this case, the attack will not be able to distinguish between members and non-members from the dominating class. 

This is most evident from the results on the 2-Purchase dataset (Fig.~\ref{fig:mia_label_auc-all}a-e in Appendix~\ref{sec:appendix-perlabel}), in which one of the two classes overwhelmingly dominates the other class (a DR of almost 1). The AUC performance of the dominant class is poor, whereas it is high for the other class, bringing the average AUC close to 0.5. This partly explains why the reported performance of MI attacks on 2-Purchase has always been comparatively poorer in the literature~\cite{ml-leaks, shokri-mia}. The per-class analysis on the remaining binary datasets is in Appendix~\ref{sec:appendix-perlabel}. 
\begin{observation}
\label{obs:dr}
If a class overwhelmingly dominates other classes, i.e., occupies a significant portion of the decision region in the feature space, then it is least susceptible to MI and SMI. An MI or SMI attack is unable to efficiently distinguish between members and non-members from this class.
\end{observation}

{\descr{Tuning Attack Models for SMI.} It may be argued that these MI attacks are not specifically trained to distinguish between members and nearby (synthetic) non-members, which may explain their poor performance in terms of SMI. We performed additional experiments where we tuned the training process of these attack models to further include nearby synthetic non-members. We observe even with tuning, the attack model is unable to achieve SMI. Details appear in Appendix~\ref{sec:tune_model}.} 


\subsection{Generalization to Other Machine Learning Models}
\label{sec:mia-models}

In this section, we demonstrate that the previous observations are not just limited to neural networks, and generalize to other machine learning models as well. More specifically, we use Logistic Regression (LR), Support Vector Machines (SVM) and Random Forests (RF) classifiers as the target classification models. Since our observations are consistent across all MI attacks,   
we only evaluate the Conf MI attack as it requires the least amount of information about the target model, making it the most portable attack between different machine learning target models. 


\descr{Results.} Figs.~\ref{fig:lr_nmvec_auc}, \ref{fig:svm_nmvec_auc}, \ref{fig:rf_nmvec_auc} display the AUCs on the original non-members from the datasets. We see that, in general, they exhibit the same as the neural network: the AUC improves as the distance of non-members from the dataset increases, with the AUC performance closer to 0.5 near the dataset. 
This trend in the AUCs is more prominent on the synthetic non-members shown in Figs~\ref{fig:lr_genvec_auc}, \ref{fig:svm_genvec_auc}, \ref{fig:rf_genvec_auc}. An interesting observation is that the AUC of the RF model is very high even for non-member vectors close to the dataset, across all datasets.
The main reason for this is that the RF model in general is more overfitted than the other models (see Table~\ref{tab:model_traintest} of Appendix~\ref{sec:model_params}).
This may seem to suggest that it is possible to launch a successful SMI attack on an RF-based target model. However, if we zoom into distances close to the training dataset, i.e., inset Fig.~\ref{fig:rf_genvec_auc}, we see that the AUC is close to 0.5 for Hamming distance $\le 2$. Thus, it is still difficult to launch an SMI attack for small distances.

   \begin{figure}[t]
     \subfloat[LR Original\label{fig:lr_nmvec_auc}]{%
       \includegraphics[width=0.48\columnwidth]{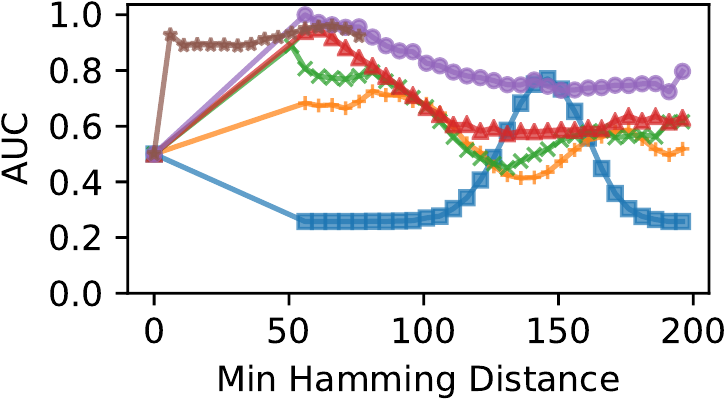}
     }\hfill
     \subfloat[LR Synthetic\label{fig:lr_genvec_auc}]{%
       \includegraphics[width=0.48\columnwidth]{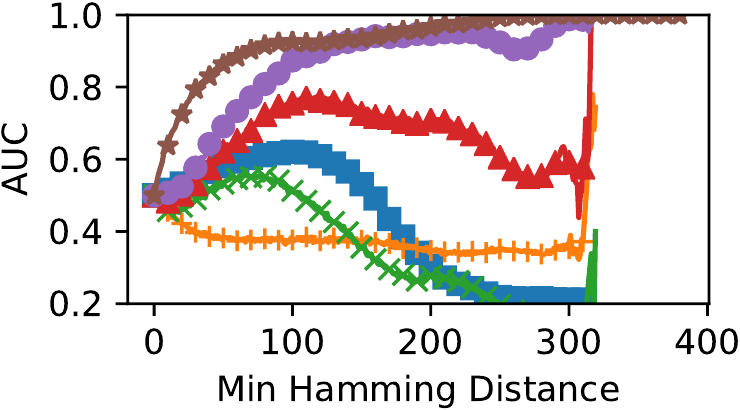}
     }\\\vspace{-2mm}
     \subfloat[SVM Original\label{fig:svm_nmvec_auc}]{%
       \includegraphics[width=0.48\columnwidth]{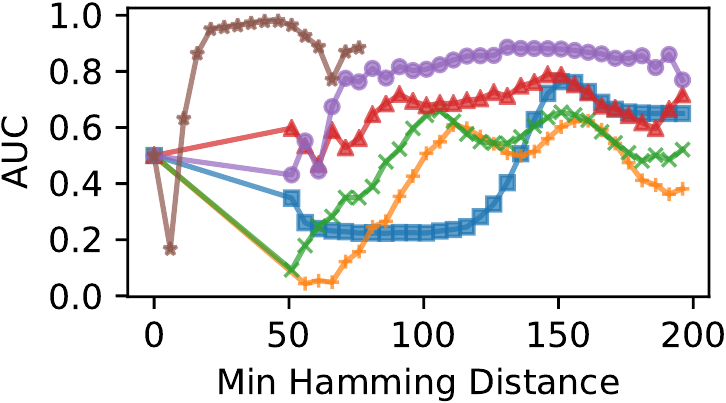}
     }\hfill
     \subfloat[SVM Synthetic\label{fig:svm_genvec_auc}]{%
       \includegraphics[width=0.48\columnwidth]{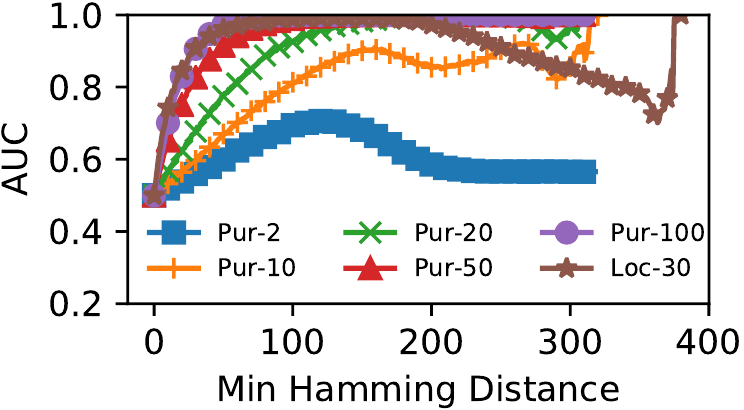}
     }\\\vspace{-2mm}
     \subfloat[RF Original\label{fig:rf_nmvec_auc}]{%
       \includegraphics[width=0.48\columnwidth]{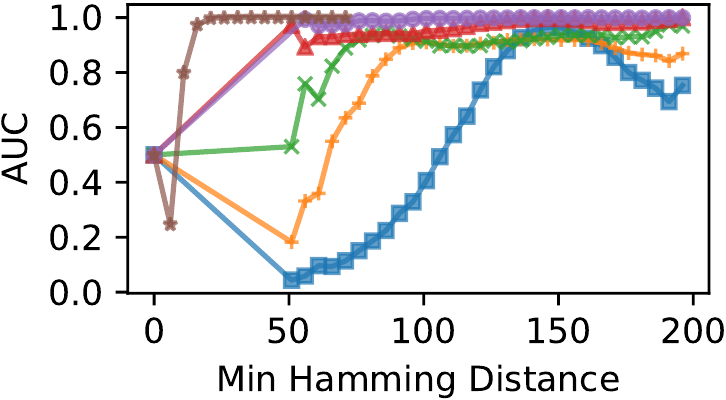}
     }\hfill
     \subfloat[RF Synthetic\label{fig:rf_genvec_auc}]{%
       \includegraphics[width=0.48\columnwidth]{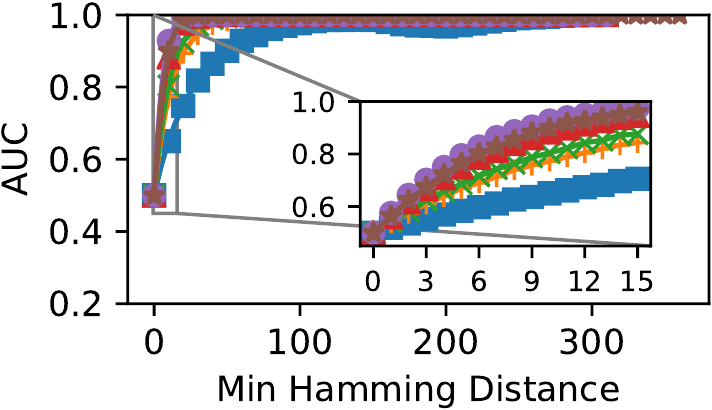}
     }
     \caption{Increasing AUC of MI with increasing Hamming distance of original and synthetic non-members from the training dataset on target models with various ML algorithms. 
     }
     \label{fig:mia_model_auc}
     \vspace{-2mm}
   \end{figure}

\begin{observation}
\label{obs:gen}
The observation that an MI attack is unable to distinguish between members and nearby non-members (strong membership inference) is consistent across different machine learning target models.
\end{observation}

\section{Attribute Inference}
\label{sec:aia-exp}

In this section, we first present the results of our experiments using the three attribute inference (AI) attacks described in Section \ref{sec:aia_how}. We show that all three AI attacks have negligible advantage in inferring the missing attributes of a target vector. On the other hand, for the same three attacks, we show that approximate attribute inference attack (AAI) advantage (Def.~\ref{def:aai-adv}) is significant, thereby suggesting that these attacks can approximately guess the missing attributes with a probability better than a random guess. We only focus on neural networks as the target model, since we have already shown that the results generalize to other machine learning models. We also study the effect of overfitting on the success advantage of both AI and AIA attacks in the last subsection.

\subsection{Attribute Inference Attacks}
\label{sec:ai-result}

To perform AI experiments (Exp.~\ref{exp:attr-inf-somesh}), we train the model exactly as described in Section~\ref{sec:infer-method}. We then (a) randomly select a member of the training set, or a non-member (from the testing set), (b) we mask a select number of most informative feature values as determined by mRMR~\cite{peng-mrmr} on the entire dataset to create the set $S$ of unknown features (15 binary features for Location and Purchase; 5 continuous features for CIFAR datasets), (c) and generate all possible siblings of the vector under $S$ (2 \blue{value bins} per feature for Location and Purchase, and up to 10 \blue{value bins} per feature for CIFAR). We then evaluate the AI attacks by giving each of the generated siblings to the underlying MI attack, and flagging those siblings that the corresponding MI attack identifies as a member vector. 
Again, the decision to use the most informative features from mRMR is to improve the likelihood of success for AI, as differences in the most informative features are likely to have the largest influence on the output of the classification model.
We determine the AI attack to be successful, if the original member vector is in this set of \emph{flagged siblings}. If there are more than one flagged sibling (excluding the original vector), we treat it as a tie and regard the attack as only partially successful. We add a fraction (determined by the number of ties) to its success count. For instance, 1/100 if there is a tie between 100 candidates. We then compute the AI advantage as the difference in the success counts between members, and non-members divided by the total counts of the tested members and non-members, respectively.
We note that we also performed Exp.~\ref{exp:attr-inf-somesh} on a single missing feature (as is done in other works~\cite{somesh-overfit, jayaraman2019evaluating}). The results are shown in Appendix.~\ref{sec:1feat_ai}. For this section, we focus on the expanded number of missing features, which is a more general case. The results for single feature AI, as we shall see, are only slightly better than multiple missing features.

\begin{table}[t]
\caption{Attribute Inference (Exp.~\ref{exp:attr-inf-somesh}) Advantage, where the adversary seeks to infer the exact attributes. The results below are normalized when dealing with ties.}
\label{tab:ai_adv}
\resizebox{\columnwidth}{!}{%
\begin{tabular}{|r|cccccccc|}
\hline
\textbf{AI}  & Loc-30 & Pur-2 & Pur-10 & Pur-20 & Pur-50 & Pur-100 & CIF-20  & CIF-100 \\ \hline
Conf  & 7.78E-4 & \ 1.38E-5   & -3.69E-4   & 2.16E-4    & 2.00E-3    & 1.65E-3     & -3.32E-7 & 4.14E-7  \\
Loss   & 7.76E-4 & -9.79E-5  & \ 5.57E-3    & 6.69E-3    & 4.59E-3    & 5.09E-3     & \ 3.33E-4  & 7.80E-4  \\
Shadow & 8.00E-4 & -2.00E-4  & \ 2.17E-3    & 2.63E-3    & 4.10E-3    & 4.20E-3     & \ 2.26E-4  & 7.99E-4  \\ \hline
\end{tabular}
}
\end{table}

   \begin{figure}[t]
   \centering
     \subfloat[Loss MI - Location and Purchase, 15 hamming distance.\label{fig:nn_yeom_f15}]{%
       \includegraphics[width=0.48\columnwidth]{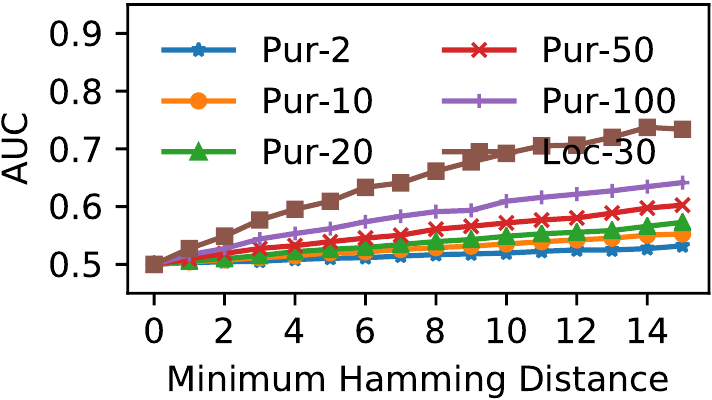}
     }\hfill
     \subfloat[CIFAR-100, zoomed to 0.5 Manhattan distance.\label{fig:nn_cifar_f05}]{%
       \includegraphics[width=0.48\columnwidth]{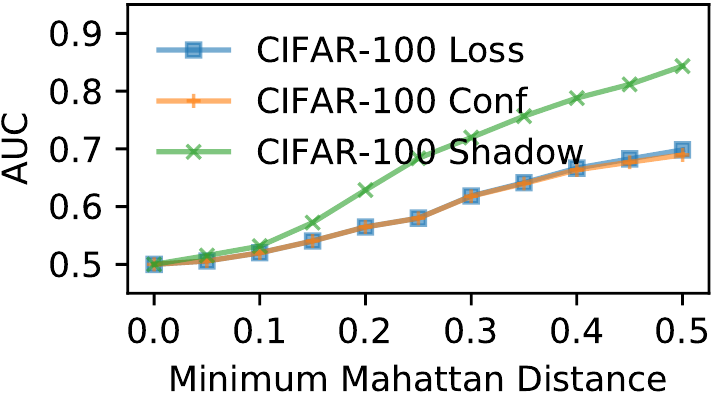}
     }
     \caption{Closer inspection of Hamming and Manhattan distance for select datasets and MI attacks previously seen in Fig.~\ref{fig:mia_genvec_auc}. Note at small distances from the training vectors, the AUC is close to 0.5, suggesting a poor AI attack.}
     \label{fig:mia_genvec_auc_zoom}
   \end{figure}

\descr{Results.} Across all attacks, we observe negligible AI advantages irrespective of the dataset and the attack (see Table \ref{tab:ai_adv}). Moreover, the advantages are also very low for more overfitted target models (Location-30, Purchase-50, Purchase-100). This suggests that an AI attack is difficult to launch, even though the same target model and datasets are susceptible to MI attacks. Our conclusion runs counter to the results from Yeom et al. on the success of attribute inference~\cite{somesh-overfit}, who demonstrate that on regression problems, a Loss AI attack can successfully infer attributes (using Loss MI attack as a subroutine), and the more overfit the target model, the more successful the attack. But this is easily reconciled by noting that our results apply to the classification problem, where the true label given to the attacker is discrete (class label). This is in contrast to the regression problem, where the true label (response) is a continuous value. The latter provides more information to the attack algorithm, which can be employed to launch a loss-based attack, i.e., Loss AI. The link to overfitting merits further exploration, and we defer this to Section~\ref{sec:ai-overfit}.    




A closer look at the Location dataset sheds more light on the reasons behind the failure of the AI attack. Previously, in Section~\ref{sec:mia-synthetic}, we observed that the performance of the Loss MI attack on the Location dataset reaches AUC greater than $\ge0.7$, significantly higher than other datasets. In  Fig.~\ref{fig:nn_yeom_f15} we focus on the Loss MI attack on non-members at Hamming distances $1$ to $15$   from the dataset. We can see that the AUC reaches $0.7$ at Hamming distance $10$ but remains close to $0.5$ between distance $1$ to $3$.
Thus, while the Loss MI attack should easily be able to discard siblings of the original vector at Hamming distances greater than $10$, it fails at closer distances and thereby resulting in an overall negligible advantage for the corresponding AI attack.   
The same reasoning applies to the CIFAR-100 dataset (Fig.~\ref{fig:nn_cifar_f05}), although under Manhattan distance.



\begin{observation}
\label{obs:no-ai}
It is difficult to infer (exact) attributes of a target vector in the training dataset from a machine learning model trained for a classification task, even if it is susceptible to membership inference.
\end{observation}

\subsection{Approximate Attribute Inference Attacks}
\label{sec:app-ai-result}

\begin{table}[t]
\caption{Approximate AI Advantage (Def.~\ref{def:aai-adv}), where the adversary seeks to infer approximate attributes ($\alpha=7.5$ for Location and Purchase, $\alpha=3.33$ for CIFAR). Results with ties are normalized.}
\label{tab:approx_ai_adv}
\resizebox{\columnwidth}{!}{%
\begin{tabular}{|r|cccccccc|}
\hline
\textbf{AAI}  & Loc-30 & Pur-2 & Pur-10 & Pur-20 & Pur-50 & Pur-100 & CIF-20  & CIF-100 \\ \hline
Conf  & 0.1609 & 0.0366 & 0.0516 & 0.0502 & 0.0958 & 0.1307 & -0.0004 & 0.0016 \\
Loss   & 0.1030 & 0.0125 & 0.0516 & 0.0541 & 0.0789 & 0.1012 & 0.0300  & 0.0325 \\
Shadow & 0.0554 & 0.0054 & 0.0067 & 0.0149 & 0.0766 & 0.0964 & 0.0339  & 0.0445 \\ \hline
\end{tabular}
}
\end{table}

Since an MI attack starts performing better as the distance of non-member vectors from the dataset increases, this suggests that the relaxed notion of approximate attribute inference (AAI) defined in Exp.~\ref{exp:app-attr-inf} may be realizable in practice. Recall that an AAI adversary is given a portion $\mathbf{x^*}$ of a vector $\mathbf{x}$, and is asked to return a vector $\mathbf{x}'$ such that $d(\mathbf{x}, \mathbf{x}') \le \alpha$, where the parameter $\alpha$ determines closeness to the exact attributes. In this section, we evaluate AAI attacks. These are essentially AI attacks, but the success is determined by the parameter $\alpha$. To set an appropriate value of $\alpha$, we need to take into account any algorithm that randomly guesses the missing features without even using the output of the classifier. Over all challenge vectors, the average distance of the guessed vectors from the target vectors will approach the expected distance of a vector $\mathbf{x}'$ from $\mathbf{x}$ whose missing features are randomly generated.  We therefore set $\alpha$ equivalent to this expected distance. This means that any algorithm that successfully guesses more the missing features within an $\alpha$ distance of the target vector is non-trivial. Note that guessing missing features trivially due to correlations in the data distribution is already covered by the way our AAI definition is constructed, i.e., learning via the model versus via the distribution.
Thus, for the Location and Purchase datasets, where we have $15$ unknown features, we set $\alpha$~=~$7.5$, and for the CIFAR dataset, with $5$ unknown continuous features (normalized between $-1$ and $1$), we set $\alpha$~=~$3.33$, which is the average distance of a random guess from the original values
\iffull
(See Appendix~\ref{app:misc}).
\fi

\descr{Results.} Table~\ref{tab:approx_ai_adv} shows the AAI advantage (Def.~\ref{def:aai-adv}) of the three AI attacks on all datasets. Overall, the AAI advantage is considerably higher than the AI advantage (from Table~\ref{tab:ai_adv}), reaching up to $0.1609$ for the Loss AI attack on the Location dataset. However, the advantage obtained is still lower than the theoretical maximum of $1$. Furthermore, the advantage is higher for more overfitted datasets, i.e., Location, Purchase-50, Purchase-100, and CIFAR-100. This indicates that increasingly the level of overfitting may improve the attack accuracy, which we shall explore in the next section. Interestingly, Shadow AI either performs worse or comparable to Conf AI and Loss AI, even though the latter attacks have less information available to them. 
The advantages seen in Table~\ref{tab:approx_ai_adv} exceed AI with one missing feature (See Appendix~\ref{sec:1feat_ai}), despite the increased inference difficulty, with more missing features.

\blue{Like Yeom et al. \cite{somesh-overfit}, our current evaluation, regards the measure of success as an adversary's ability to infer attributes with a single guess, reported as an average over multiple vectors; However, we acknowledge there are additional measures of success. For example \emph{top-k}, whereby an attacker has the opportunity to submit their top $k$ guesses.}


\begin{observation}
\label{obs:aai}
It is possible to infer attributes approximately close to their true values with a success rate significantly greater than random guess when the target model is susceptible to membership inference. 
\end{observation}


\subsection{AI, AAI and Relation to Overfitting}
\label{sec:ai-overfit}

\begin{table}[t]
\caption{Approximate AI (Exp.~\ref{exp:app-attr-inf}) Advantage, where the Shadow adversary seeks to infer approximate attributes ($\alpha=7.5$) from various states of generalized Purchase-100 Models}
\label{tab:aia-overfit}
\centering
\resizebox{\columnwidth}{!}{%
\begin{tabular}{|r|ccccccc|}
\hline
Dataset Size  & 20K & 40K & 60K & 80K & 100K & 150K & 200K \\
Overfitting   & 0.368 & 0.301 & 0.271 & 0.251 & 0.237 & 0.211 & 0.193 \\
Shadow AI & 0.0024 & 0.0046 & 0.0021 & 0.0052 & 0.0040 & 0.0049 & 0.0033 \\
Shadow AAI & 0.118 & 0.098 & 0.096 & 0.078& 0.066 & 0.046 & 0.026\\ \hline
\end{tabular}
}
\end{table}


In both AI and AAI attacks, we observed greater advantage on more overfitted target models. To explore this further, we focus on the Purchase-100 dataset and the Shadow AI attack. We define the overfitting level of a model as the generalization error (GE) as defined in Eq.~\ref{eq:gen-err}. To alter GE, and hence the degree of overfitting, we vary the amount of training data, while maintaining proportional splits between training and testing sets. As we increase the training data size from 20,000 (20K) to 200,000 (200K), the generalization error  decreases from 0.368 down to 0.193 as shown in Table~\ref{tab:aia-overfit}.





\descr{Results.} From the ``Shadow AI'' row of Table~\ref{tab:aia-overfit}, we can see that increasing the overfitting level has little to no impact on the AI advantage (the Shadow AI result in Table~\ref{tab:ai_adv} corresponds to a dataset size of 40K). Returning to the comparison with the findings of Yeom et al. on the effectiveness of AI on regression tasks in Section~\ref{sec:ai-result}, our results indicate that for a classification problem, AI remains ineffective even if we increase the degree of overfit.   
On the other hand, there is a positive correlation between overfitting level and the AAI advantage, evident from the row labeled ``Shadow AAI'' in Table~\ref{tab:aia-overfit}. As the overfitting level increases from 0.193 up to 0.368, the AAI advantage improves from 0.026 to 0.118.



\begin{observation}
The more overfitted a target classification model, the more susceptible it is to approximate attribute inference. On the other hand, attribute inference remains hard even with increased overfitting levels.  
\end{observation}




\section{Related Work}
\label{sec:rw}
\vspace{-2mm}


The three black-box MI attacks evaluated in this paper were proposed by Shokri et al.~\cite{shokri-mia}, Salem et al.~\cite{ml-leaks} and Yeom et al.~\cite{somesh-overfit}. All three works have used a split of a real dataset into training and testing sets, and demonstrated the effectiveness of MI using the testing sets. We have shown that most vectors in the testing set, i.e., non-members, are expected to be far from the training set, which explains why the relationship of MI performance to distance from members was not identified in these works. We have also shown that our results apply in the white-box setting, by evaluating the MI attacks from Nasr et al.~\cite{nasr2018comprehensive}, who proposed passive and active white box attacks targeting both standalone and federated models. Of course, the research on MI is not limited to these works. For instance, in~\cite{hayes2019logan} black and white box MI attacks are evaluated on generative adversarial networks; in~\cite{merlin} a new MI attack is proposed based on the loss-based MI attack from Yeom et al. evaluated in our paper, and in~\cite{disparate-sub-group} the authors show that even if MI attacks are ineffective as a whole on a dataset, they have disparate effectiveness on different sub-groups in the dataset. 
We have already demonstrated that our observations generalize to other MI attacks and models, since the underlying principle remains the same, i.e., ML models are less susceptible to strong membership inference in the classification setting.    





The central theme of our paper is on the feasibility of attribute inference, also known as model inversion~\cite{model-inversion, pharma, yang2019neural, zhao2019adversarial}. A criticism of these works on model inversion is that they essentially exploit the correlation between the attributes and the true label, to infer the missing attributes~\cite{shokri-mia}. Finding such correlations is the very purpose of the learning task, and therefore, the missing attributes would be learned regardless of whether the challenge vector is a member or a non-member~\cite{shokri-mia}. The model inversion or attribute inference definition from Yeom et al.~\cite{somesh-overfit} avoids this issue by defining the AI advantage as the difference between inferring attributes with the model and without the model (i.e., through the distribution). Indeed, our definitions of AI and AAI use the same approach, based on their work. Yeom et al.~\cite{somesh-overfit} are also the first to formally relate MI attacks to AI attacks. They also formalise the role of overfitting to the effectiveness of MI and AI attacks, a link which was previously experimentally identified and demonstrated in~\cite{shokri-mia, ml-leaks}. 
As mentioned previously, they demonstrate that AI attacks are feasible on regression problems, with the accuracy of the attacks improving with the level of overfit. Although the AI attack performance is not as significant as the MI attack, it is still quite substantial reaching an advantage of up to 0.5 on one of the datasets~\cite{somesh-overfit}. We have shown that for classification problems, only approximate attribute inference seems to be feasible. 
Apart from~\cite{somesh-overfit}, Jayamaran and Evans~\cite{jayaraman2019evaluating} have also experimentally evaluated attribute inference attacks on classification models. Even though the goal of their analysis is to evaluate privacy leakage from classification models treated with differential privacy, their results with lower privacy (higher values of the privacy parameter $\epsilon$~\cite{dwork2014algorithmic}) can be considered as closer to the non-private setting. These results also show low AI advantages as compared to MI attacks, although the authors do not delve into the reasons.

Another related area is the investigation of factors effecting membership inference. Sablayrolles et al.~\cite{sablayrolles2019white} seek the optimal strategy for membership inference and find that such a strategy depends only on the loss function, implying that, asymptotically, knowledge of the model parameters (white box setting) does not provide any benefits over black box access. However, their treatment does not explore distance-based impact on membership inference as is done in our work. Long et al.~\cite{longpragmatic} explore the performance of membership inference focused on training data records which are more vulnerable, in contrast to looking at membership inference performance as an aggregate over the entire training dataset. They find that records which have fewer neighbors are more vulnerable, as their presence or absence has more influence on the model's output. They also state that it is difficult for an MI attack to distinguish between a member and its non-member neighbors. Unlike \cite{longpragmatic}, we formally prove the distinction between MI and SMI, and how this separation negatively impacts AI (and AAI) on classification models.




On the definitional side, Wu et al.~\cite{wu-influence} present an initial formal definition of attribute inference as the difference in inferring from the output of the model versus through the distribution (without access to the model). The definition from Yeom et al.~\cite{somesh-overfit}, which is the basis of our related definition, follows the same line of thinking. In addition to membership and attribute inference, Melis et al.~\cite{melis-exploit} also consider \emph{property inference}, which is a property of a subset of training points within a class but not true of the entire class. They show that it is possible to infer properties that are independent of what characterizes the class through unintended learning by the machine learning algorithm. Unlike membership or attribute inference which is tied to individual data points, their property inference relates to multiple training points (subsets).



This is similar to other attacks on machine learning models, such as \emph{model extraction}~\cite{tramer-stealing}, which apply to the entire model itself and not necessarily to individuals in the training dataset. In a model extraction attack, unknown parameters of the model are retrieved to construct similarly behaving models (hence stealing the model in a proprietary sense). On the defense side, it has been demonstrated that MI and AI attacks can be mitigated by the use of \emph{differential privacy}~\cite{dwork2014algorithmic, jayaraman2019evaluating, abadi2016deep}, although, this comes at a potential loss in utility~\cite{jayaraman2019evaluating, farokhi2020modelling, zhao2019adversarial}.
Our findings on the infeasibility of AI attacks indicate that we may only need protection against (the weaker) approximate attribute inference, for which tailored differentially private learning algorithms can be constructed offering better utility. This is particularly useful for applications where membership inference is less of a concern, or may even be desirable. A case in point being machine learning auditors, based on membership inference attacks, to prevent unauthorized use of personal data~\cite{miao2019audio, song2019auditing}. {Additionally only evaluating defenses against AI may mask potential privacy leakage though AIA, an arguably simpler attack and thus a more difficult task to defend.}

Finally Adversarial examples are vectors with applied perturbations close to the original target that result in large variations in the model's behavior, commonly observed as a mis-prediction~\cite{goodfellow2014explaining}. In the setting of MI or AI, given an adversarial example of a vector within the training dataset, the large difference between the behavior of the known and adversarial example would allow for their distinction. However, as Long et al.~\cite{longpragmatic} state, the majority of the neighborhood around the vector would have a minimal difference on the model output; with the adversarial example behaving as an exception, rather than the norm. Though combative methods have been developed to train models robust to adversarial examples ~\cite{gu2014towards}, we speculate that robust adversarial models will only have a minor positive impact on the mitigation of the MI/AI attack, as robust models should preserve the regular behavior of the model, to only mitigate the behavior of the adversarial examples. Though this warrants further investigation.

\section{Conclusion}

Our results show that it is infeasible for an attacker to correctly infer missing attributes of a target individual whose data is used to train a machine learning model for a classification problem owing to the inability of membership inference attacks to distinguish between members and nearby non-members. For applications, where the privacy concern is attribute inference, and not membership inference, defense mechanisms tailored to protect against approximate attribute inference can be constructed. As a future direction, it will be interesting to explore whether the approximate attribute inference attacks mentioned in this paper can be improved to infer missing attributes as close as possible to the original attributes.




\section*{Acknowledgments}
This work was conducted with funding received from the Optus Macquarie University Cyber Security Hub, in partnership with the Defence Science \& Technology Group and Data61-CSIRO, through the Next Generation Technologies Fund. Benjamin Zhao has also been funded by an Australian Government RTP scholarship. 


\bibliographystyle{plain}
\bibliography{air-ref}


\appendices

\section{Model Parameters}
\label{sec:model_params}

\begin{table}[t]
\centering
\caption{Summary of training and testing accuracies, with MI AUC for all machine learning classifiers.}
\label{tab:model_traintest}
\resizebox{\columnwidth}{!}{%
\begin{tabular}{|l||l|c|c|c||l|c|c|c|}
\hline
Dataset      & Model & Train Acc & Test Acc & MI AUC & Model - MI          & Train Acc & Test Acc & MI AUC \\ \hline
             & LR - Conf    & 1.000        & 0.582    & 0.897  & NN - Conf   & 1.000        & 0.794    & 0.705  \\ 
\multirow{2}{*}{Loc-30}  & SVM - Conf   & 1.000        & 0.731    & 0.916  & NN - Loss   & 1.000        & 0.794    & 0.710  \\ 
             & RF - Conf    & 1.000        & 0.566    & 0.975  & NN - Shadow & 1.000        & 0.666    & 0.909  \\ 
             & NN - Local & 0.998 & 0.430 & 0.891 & NN - Global & 0.998 & 0.430 & 0.886 \\
             \hline
             & LR - Conf    & 1.000        & 0.484    & 0.765  & NN - Conf   & 0.999        & 0.765    & 0.708  \\ 
\multirow{2}{*}{Pur-100} & SVM - Conf   & 1.000        & 0.799    & 0.855  & NN - Loss   & 0.999        & 0.765    & 0.720  \\ 
             & RF - Conf    & 1.000        & 0.606    & 0.998  & NN - Shadow & 1.000        & 0.700    & 0.842  \\ 
             & NN - Local & 0.538 & 0.487 & 0.508 & NN - Global & 0.538 & 0.487 & 0.719 \\
             \hline
             & LR - Conf    & 0.995        & 0.601    & 0.614  & NN - Conf   & 0.998        & 0.832    & 0.629  \\ 
\multirow{2}{*}{Pur-50}  & SVM - Conf   & 1.000        & 0.857    & 0.716  & NN - Loss   & 0.998        & 0.832    & 0.638  \\ 
             & RF - Conf    & 1.000        & 0.724    & 0.980  & NN - Shadow & 1.000        & 0.778    & 0.763  \\ 
             & NN - Local & 0.692 & 0.657 & 0.520 & NN - Global & 0.692 & 0.657 & 0.668 \\
             \hline
             & LR - Conf    & 0.973        & 0.785    & 0.552  & NN - Conf   & 0.999        & 0.889    & 0.577  \\ 
\multirow{2}{*}{Pur-20}  & SVM - Conf   & 1.000        & 0.906    & 0.584  & NN - Loss   & 0.999        & 0.889    & 0.582  \\ 
             & RF - Conf    & 1.000        & 0.813    & 0.917  & NN - Shadow & 1.000        & 0.841    & 0.690  \\ 
             & NN - Local & 0.803 & 0.781 & 0.505 & NN - Global & 0.803 & 0.781 & 0.626 \\
             \hline
             & LR - Conf    & 0.973        & 0.878    & 0.521  & NN - Conf   & 0.999        & 0.911    & 0.558  \\ 
\multirow{2}{*}{Pur-10}  & SVM - Conf   & 1.000        & 0.932    & 0.530  & NN - Loss   & 0.999        & 0.911    & 0.561  \\ 
             & RF - Conf    & 1.000        & 0.840    & 0.902  & NN - Shadow & 1.000        & 0.868    & 0.644  \\ 
             & NN - Local & 0.836 & 0.818 & 0.503 & NN - Global & 0.836 & 0.818 & 0.608 \\
             \hline
             & LR - Conf    & 1.000        & 0.986    & 0.499  & NN - Conf   & 0.998        & 0.959    & 0.521  \\ 
\multirow{2}{*}{Pur-2}   & SVM - Conf   & 1.000        & 0.987    & 0.502  & NN - Loss   & 0.998        & 0.959    & 0.522  \\ 
             & RF - Conf    & 1.000        & 0.921    & 0.781  & NN - Shadow & 0.999        & 0.944    & 0.580  \\ 
             & NN - Local & 0.914 & 0.906 & 0.505 & NN - Global & 0.914 & 0.906 & 0.567 \\
             \hline

\multirow{2}{*}{CIFAR-20}   & NN - Conf     & 0.920 & 0.322 & 0.544 & NN - Loss   & 0.920 & 0.322 & 0.799  \\
                            & NN - Shadow   & 0.999 & 0.281 & 0.925 & \multicolumn{1}{|c|}{-} & - & - & - \\ \hline
\multirow{2}{*}{CIFAR-100}  & NN - Conf     & 0.831 & 0.214 & 0.524 & NN - Loss   & 0.831 & 0.214 & 0.844  \\ 
                            & NN - Shadow   & 0.999 & 0.170 & 0.967 & \multicolumn{1}{|c|}{-} & - & - & - \\ \hline
\end{tabular}
}
\vspace{-4mm}
\end{table}

\subsection{Target Models}
We will first describe the \textbf{Neural Network (NN)} based target models used in the bulk of our experiments, followed by the configurations of the classifiers in Section~\ref{sec:mia-models}. The training and testing accuracies can be found in Table~\ref{tab:model_traintest}.
{\textbf{Location:}}
The model was trained in keras as a fully connected NN with 1 hidden layer of 128 nodes with the ``tanh'' activation function. We replicate the training and testing accuracy of \cite{shokri-mia}'s target model. 
%
{\textbf{Purchase:}}
The target model was trained in keras as a fully connected neural network with 1 hidden layer of [128] nodes with a ``tanh'' activation function. This architecture replicates the training and testing accuracy for the target model as previously reported in \cite{shokri-mia}.  
%
{\textbf{CIFAR:}}
The target model is a multilayer perceptron, consisting of two hidden layers of 256 units, with relu activation layer and a softmax output layer. This is the same architecture used in \cite{jayaraman2019evaluating}. 

  

  
\textbf{Logistic Regression (LR)}: The parameter C was set at 100 for all datasets, with all other parameters remain at the default values.
\textbf{Support Vector Machine (SVM)}: We select a linear kernel for all the datasets. We keep parameters at default values.
\textbf{Random Forest (RF)}: The number of estimators was chosen to be 100 with no depth specified, the remaining parameters were kept as defaults.

The training and testing accuracies for each algorithm, and for each datasets are noted in Table~\ref{tab:model_traintest}.


\subsection{MI Attack Configurations}
\label{sec:mia-config-how}

Due to the different data requirements for each attack, the way the data is partitioned differs, we note these differences in this section. The average MI AUC can be found in Table~\ref{tab:model_traintest}. For the Conf and Loss attacks, we do not require additional data to train an attack model.

\subsubsection{Conf and Loss attacks}

\textbf{Location:}
We take the full dataset and divide it into 2 parts. 20\% is used for training the target model and remainder 80\% is kept for testing purposes. 
\textbf{Purchase:}
We sample 20,000 records from the dataset and divide it into 2 parts. The first 80\% is used for training the target model and remaining 20\% is kept for testing purposes. 
\textbf{CIFAR:}
50,000 records are sampled from the dataset to constitute our experimental dataset, from this 20\% is reserves as the training data, and the remaining 80\% is use for testing. 

\subsubsection{Shadow MI}
\noindent\textbf{Location:}
We take the full dataset and divide it into 3 parts. The first 20\% is used for training the target model, 64\% for training the shadow models and the remaining 16\% is retained for testing. Our Shadow MI attack is from the open-source library~\cite{mia-library}. The training and testing accuracies are found in Table~\ref{tab:model_traintest}. Our models are as follows:
\begin{enumerate}[leftmargin=*]
  \item \textbf{Shadow Models:} We select 60 attack models for Location dataset, consistent with \cite{shokri-mia}. The architecture of these shadow models and the size of their training dataset are equivalent to the target model.
  \item \textbf{Attack Model:} The attack model is multilayer perceptron with a 64-unit hidden layer and a sigmoid output layer. This architecture replicates the precision and recall as previously reported in \cite{shokri-mia}. For the Location-30 dataset our MI attack obtains a precision of 0.93 and recall of 0.82
\end{enumerate}

\noindent\textbf{Purchase:}
We sample 40000 records from the dataset and divide it into 3 parts. The first 25\% is used for training the target model, 67.5\% for training the shadow models and the last 7.5\% is kept for testing. The setup for running this attack on the Purchase datasets are as follows:
\begin{enumerate}[leftmargin=*]
  \item \textbf{Shadow Models} We chose the number of shadow models as 20 for Purchase dataset. The architecture of these shadow models and the size of their training dataset are the same as the target model.
  \item \textbf{Attack Model} The attack model is multilayer perceptron with a 64-unit hidden layer and a sigmoid output layer. This architecture replicates the precision and recall observed in \cite{shokri-mia}. We obtain precision of 0.66, 0.78, 0.81, 0.85, 0.89 and recalls of 0.54, 0.57, 0.6, 0.67, 0.76 for Purchase-2, 10, 20, 50, 100, respectively.
\end{enumerate}

\noindent\textbf{CIFAR:}
We sample complete dataset(around 50000 records) from the dataset and divide it into 3 parts. The first 20\% is used for training the target model,next 72\% for training the shadow model and the rest 8\% is kept for testing purposes. The setup for running this attack on this dataset is as follows:
\begin{enumerate}[leftmargin=*]
  \item \textbf{Shadow Models} We chose the number of attack models as 5 for CIFAR dataset which is the same as \cite{jayaraman2019evaluating}. The architecture of this shadow model and the size of the training dataset is the same as the target model.
  \item \textbf{Attack Model} A multilayer perceptron (two 64 unit hidden layer with ``tanh'' activation layer and a sigmoid output layer). This architecture matches the precision and recall of the attack model previously reported in \cite{shokri-mia}. We achieve 0.98 precision and 0.9 recall for CIFAR-100.
\end{enumerate}

\subsection{Local and Global White Box Inference Attacks~\cite{nasr2018comprehensive}}
\label{sec:appendix-wb-how}
As a result of the federated setting, the target models for our datasets differ.
The target models and attack model architecture, as well as the training and testing setup, originally described by \cite{nasr2018comprehensive} are utilized in this study. 

\textbf{Target Model}
Our target model for both datasets consisted of five layers (1024, 512, 256, 128, 100) with “tanh” activation, replicated from \cite{nasr2018comprehensive}. 
Each party as well as the server is trained on this model across 100 epochs with an Adam optimizer with learning rate of 0.0001 and cross entropy loss.

\textbf{Attack Model}
The attack model takes in a number of different inputs from the target model, which are trained on 'submodules' before being combined in a final network. These inputs described below, with c being equal to the number of classes of the dataset: 
\begin{itemize}[leftmargin=*]
    \item Gradient loss of the final layer - One convolutional layer (1000) with kernel size (1, c) and three hidden layers (1024, 512, 128) 
    \item One hot encoded true label - 2 hidden layers (128, 64)  
    \item Predicted labels - 2 hidden layers (100, 64)  
    \item Output for the correct label – 2 hidden layers (c, 64)
\end{itemize}

The combined input is trained using three hidden layers (256, 126, 64, 1). 
"ReLu" activation is used throughout the attack model, with an Adam optimiser with learning rate of 0.00001 and mean square error loss. 

\textbf{Datasets}
During target model training the Location and Purchase datasets were both split with 20\% (30,000 for Purchase, 1,158 for Location) used for the initial target model training, and 80\% (150,000 for Purchase, 5,790 for Location) for testing, as described for the purchase dataset in \cite{nasr2018comprehensive}). 
The data was further split equally amongst the three parties so that each party had a training and testing set of the same size.  
The attack model was subsequently trained with half of the original training data and the same amount of the original testing data (representing members and nonmembers, respectively). 
Each batch was designed to have 50\% of members and nonmembers.
The remaining samples were used for testing.







    \newcommand{\figwide}{0.16}

   \begin{figure*}[t]
   \centering
     \subfloat[Loc-30 Conf MI\label{fig:nn_salem_label_auc-loc30}]{%
       \includegraphics[width=\figwide\textwidth]{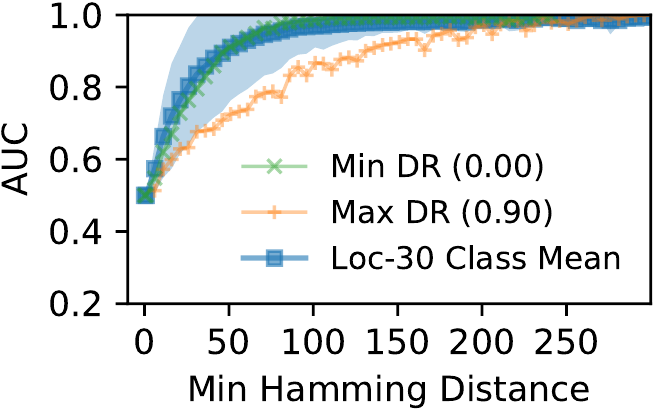}
     }
     \subfloat[Loc-30 Loss MI\label{fig:nn_yeom_label_auc-loc30}]{%
       \includegraphics[width=\figwide\textwidth]{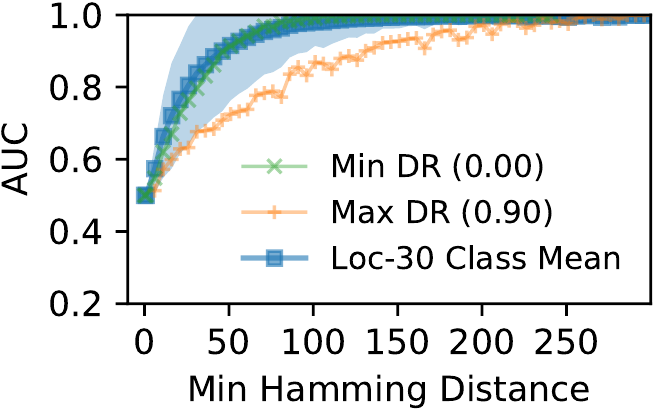}
     }
     \subfloat[Loc-30 Shadow MI\label{fig:nn_shokri_label_auc-loc30}]{%
       \includegraphics[width=\figwide\textwidth]{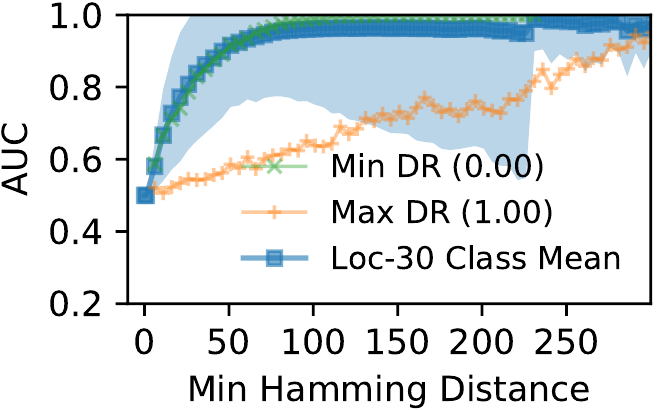}
     }
     \subfloat[Loc-30 Local WB\label{fig:nasrloc_label_auc-loc30}]{%
       \includegraphics[width=\figwide\textwidth]{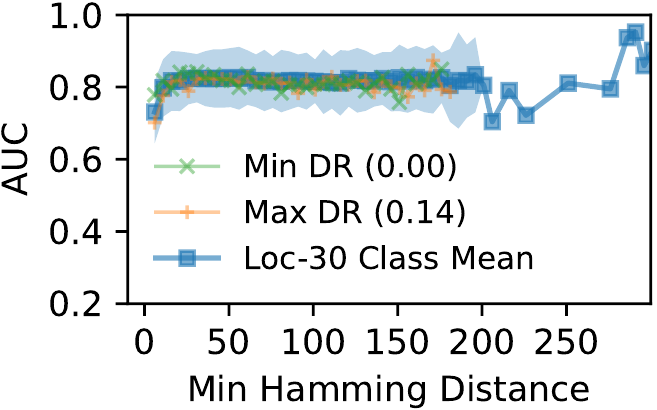}
     }
     \subfloat[Loc-30 Global WB\label{fig:nasrglob_label_auc-loc30}]{%
       \includegraphics[width=\figwide\textwidth]{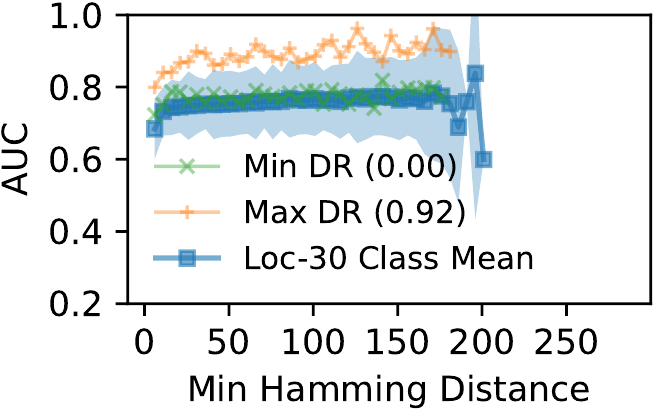}
     }\\ \vspace{-2mm}
     \subfloat[Pur-2 Conf MI\label{fig:nn_salem_label_auc-pur2}]{%
       \includegraphics[width=\figwide\textwidth]{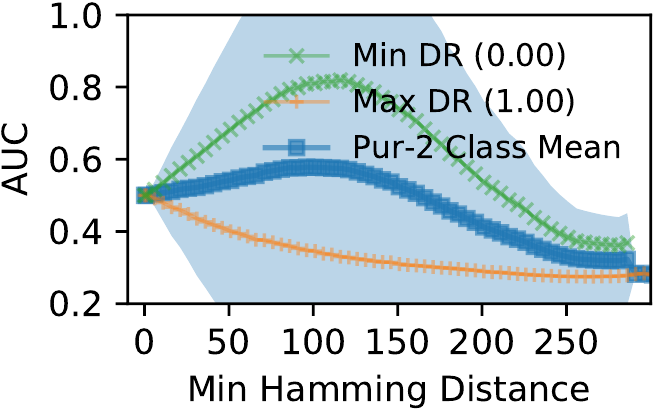}
     }
     \subfloat[Pur-2 Loss MI\label{fig:nn_yeom_label_auc-pur2}]{%
       \includegraphics[width=\figwide\textwidth]{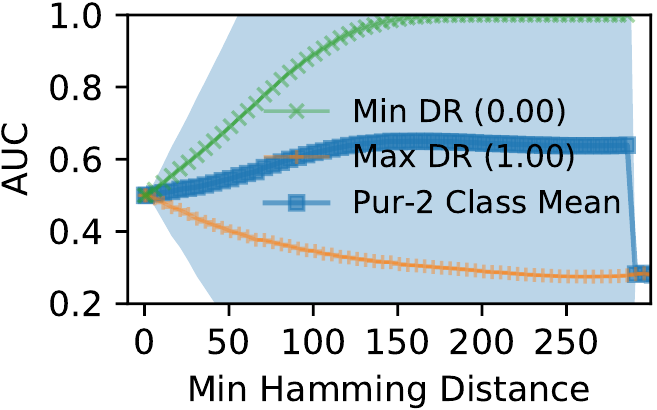}
     }
     \subfloat[Pur-2 Shadow MI\label{fig:nn_shokri_label_auc-pur2}]{%
       \includegraphics[width=\figwide\textwidth]{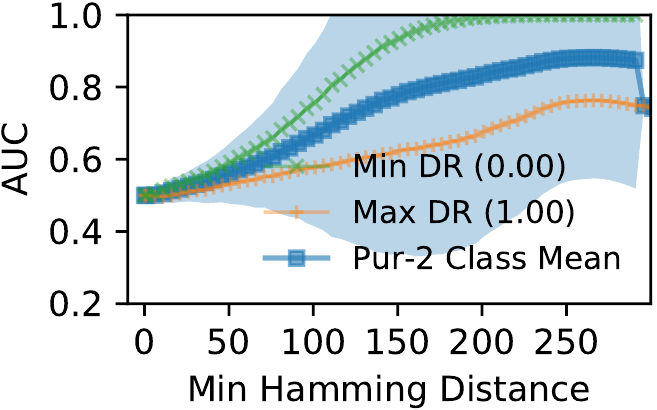}
     }
     \subfloat[Pur-2 Local WB\label{fig:nasrloc_label_auc-pur2}]{%
       \includegraphics[width=\figwide\textwidth]{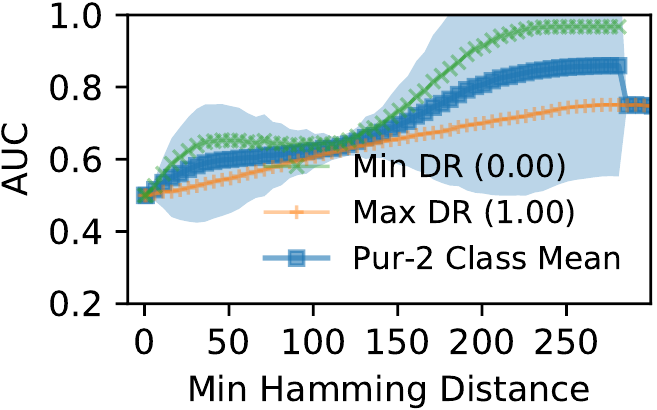}
     }
     \subfloat[Pur-2 Global WB\label{fig:nasrglob_label_auc-pur2}]{%
       \includegraphics[width=\figwide\textwidth]{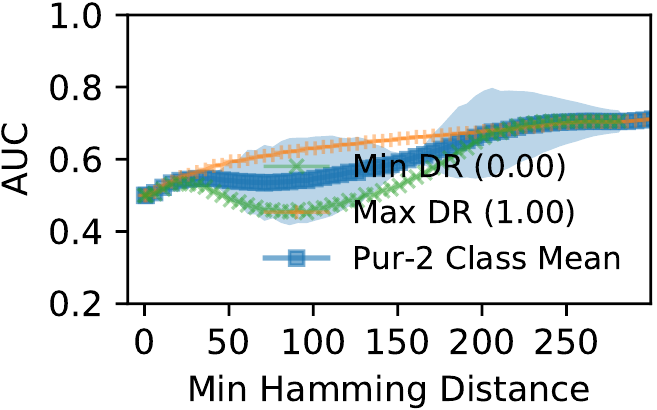}
     }\\ \vspace{-2mm}
     
     \subfloat[Pur-10 Conf MI\label{fig:nn_salem_label_auc-pur10}]{%
       \includegraphics[width=\figwide\textwidth]{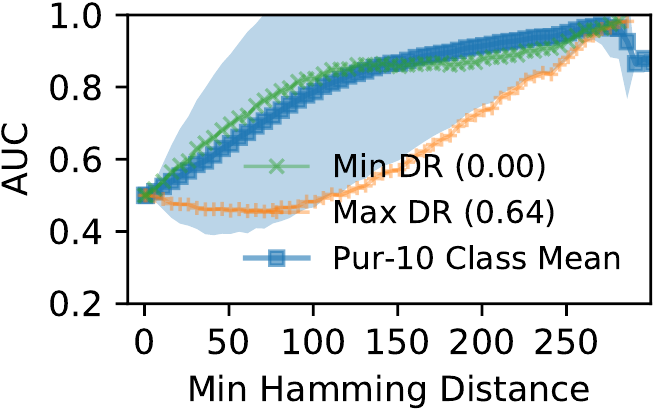}
     }
     \subfloat[Pur-10 Loss MI\label{fig:nn_yeom_label_auc-pur10}]{%
       \includegraphics[width=\figwide\textwidth]{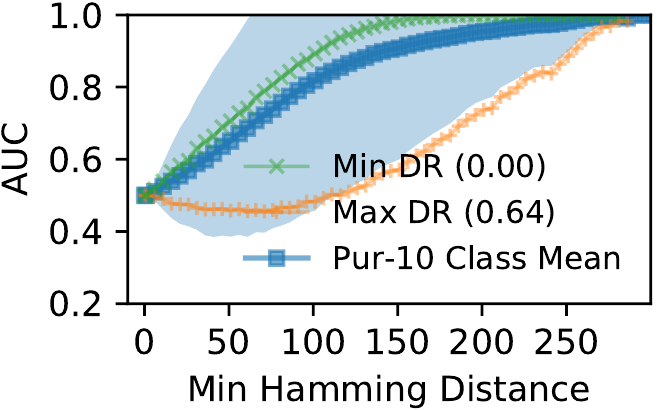}
     }
     \subfloat[Pur-10 Shadow MI\label{fig:nn_shokri_label_auc-pur10}]{%
       \includegraphics[width=\figwide\textwidth]{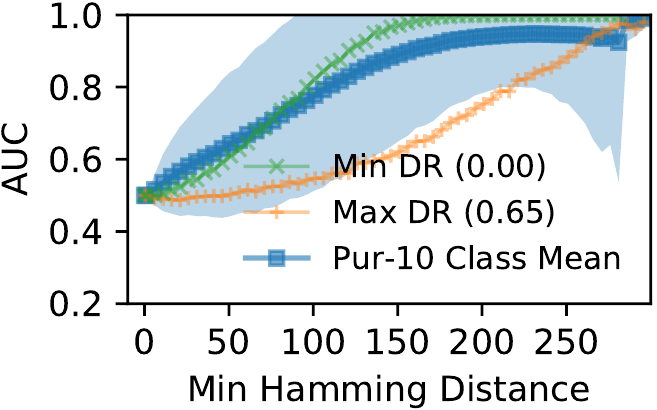}
     }
     \subfloat[Pur-10 Local WB\label{fig:nasrloc_label_auc-pur10}]{%
       \includegraphics[width=\figwide\textwidth]{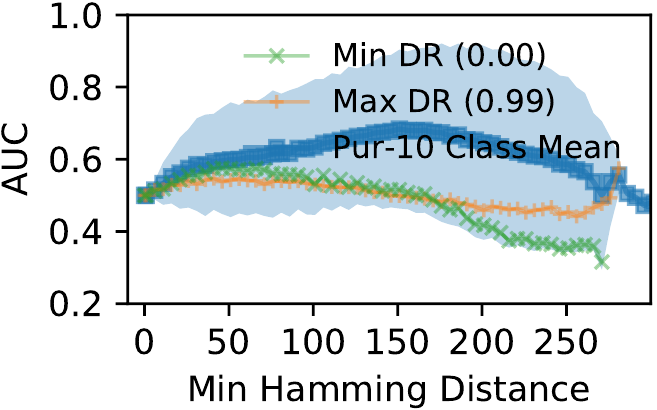}
     }
     \subfloat[Pur-10 Global WB\label{fig:nasrglob_label_auc-pur10}]{%
       \includegraphics[width=\figwide\textwidth]{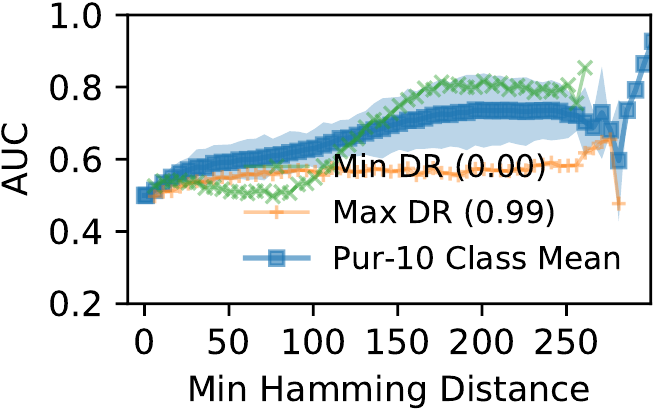}
     }\\ \vspace{-2mm}
     
     \subfloat[Pur-50 Conf MI\label{fig:nn_salem_label_auc-pur50}]{%
       \includegraphics[width=\figwide\textwidth]{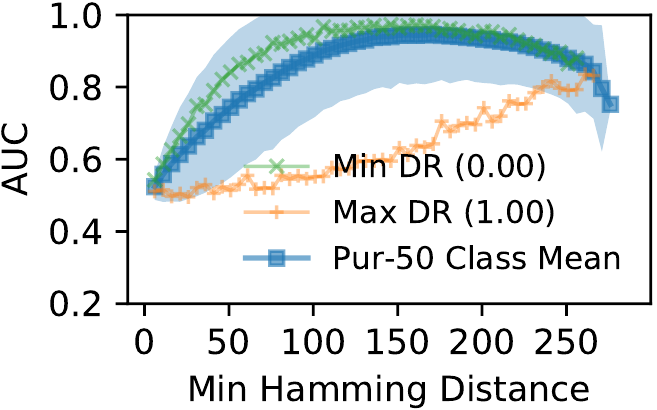}
     }
     \subfloat[Pur-50 Loss MI\label{fig:nn_yeom_label_auc-pur50}]{%
       \includegraphics[width=\figwide\textwidth]{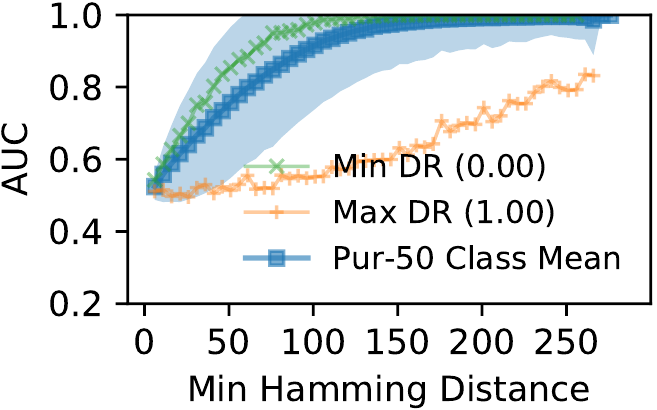}
     }
     \subfloat[Pur-50 Shadow MI\label{fig:nn_shokri_label_auc-pur50}]{%
       \includegraphics[width=\figwide\textwidth]{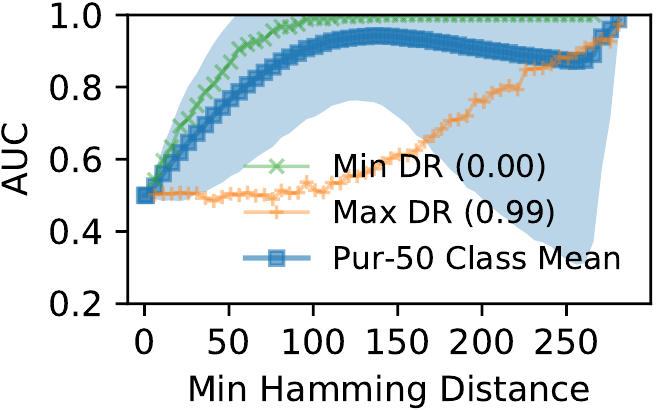}
     }
     \subfloat[Pur-50 Local WB\label{fig:nasrloc_label_auc-pur50}]{%
       \includegraphics[width=\figwide\textwidth]{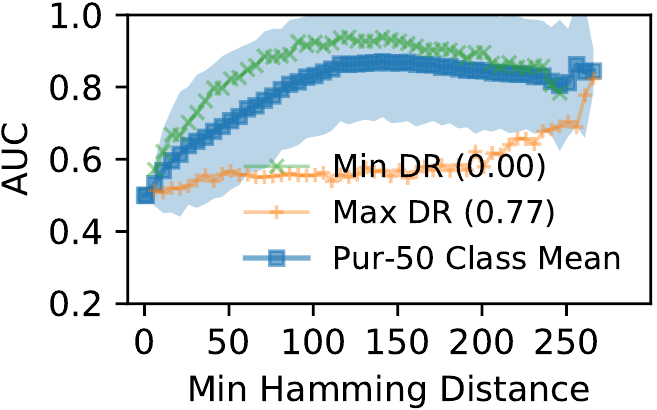}
     }
     \subfloat[Pur-50 Global WB\label{fig:nasrglob_label_auc-pur50}]{%
       \includegraphics[width=\figwide\textwidth]{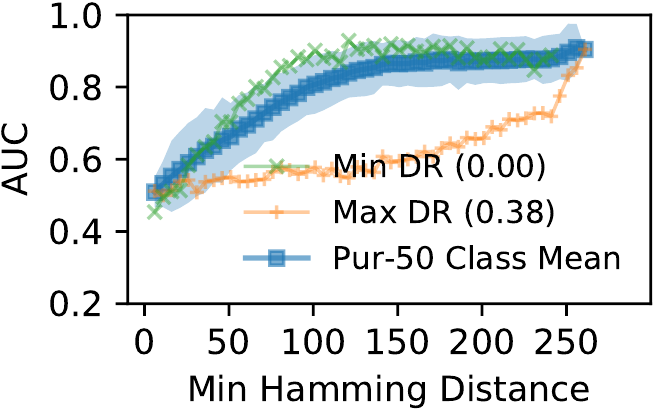}
     }\\ \vspace{-2mm}
     
     \subfloat[Pur-100 Conf MI\label{fig:nn_salem_label_auc-pur100}]{%
       \includegraphics[width=\figwide\textwidth]{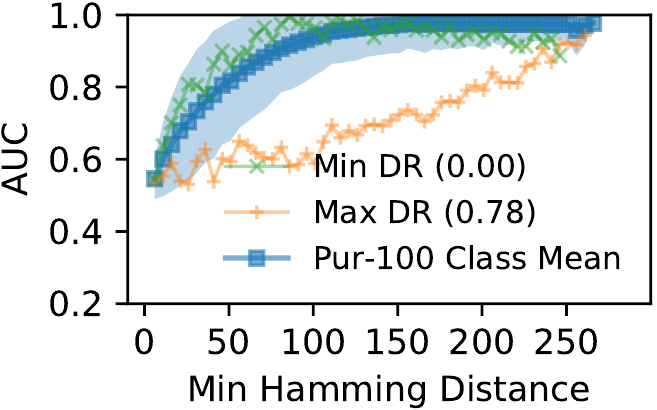}
     }
     \subfloat[Pur-100 Loss MI\label{fig:nn_yeom_label_auc-pur100}]{%
       \includegraphics[width=\figwide\textwidth]{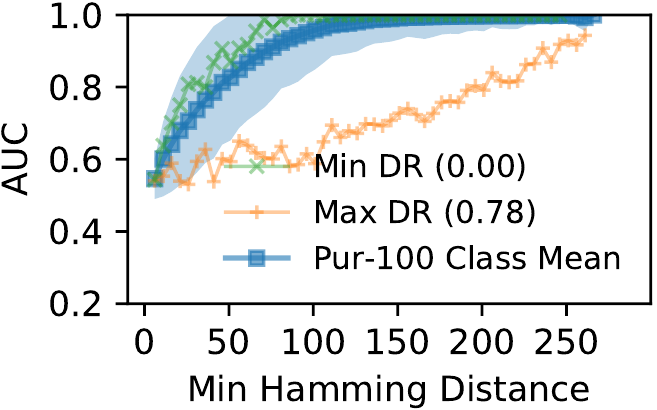}
     }
     \subfloat[Pur-100 Shadow MI\label{fig:nn_shokri_label_auc-pur100}]{%
       \includegraphics[width=\figwide\textwidth]{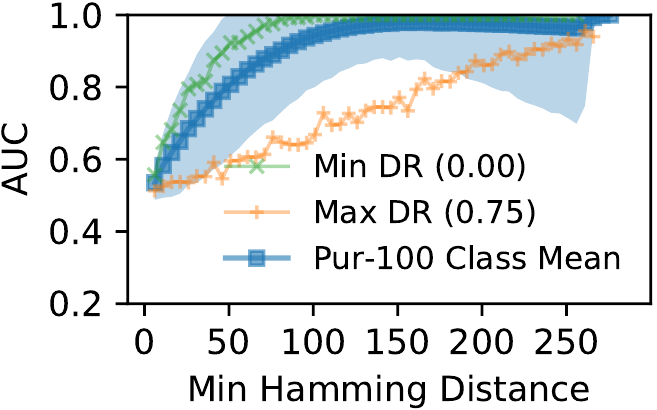}
     }
     \subfloat[Pur-100 Local WB\label{fig:nasrloc_label_auc-pur100}]{%
       \includegraphics[width=\figwide\textwidth]{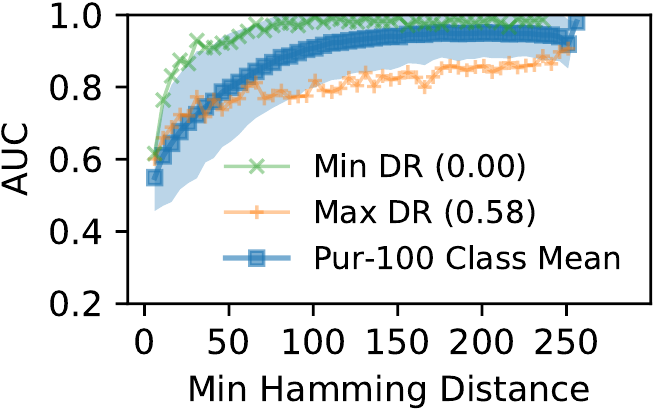}
     }
     \subfloat[Pur-100 Global WB\label{fig:nasrglob_label_auc-pur100}]{%
       \includegraphics[width=\figwide\textwidth]{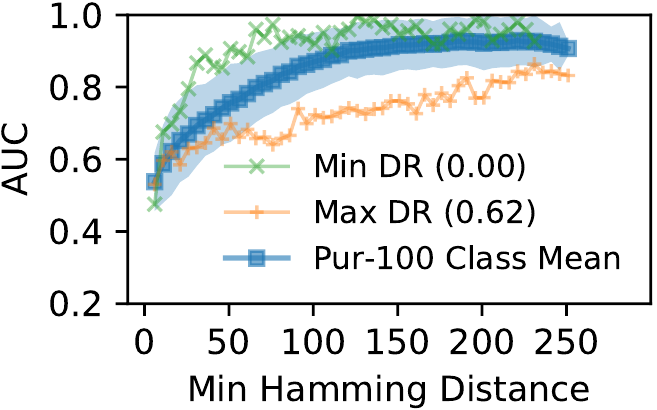}
     }
     \caption{Increasing AUC of MIA with increasing distance of synthetic non-members from the training dataset, with a separation of class labels depending on the size of the DR, for the \textbf{Loc-30}, \textbf{Pur-2, 10, 20, 50, 100} datasets.}
     \label{fig:mia_label_auc-all}
     \vspace{-3mm}
   \end{figure*}

\section{Additional Figures and Experimentation}
\subsection{Additional Plots}
\textbf{CIFAR-20 Plots}
\label{sec:appendix-cifar20}
In Section~\ref{sec:mia-exp}, we presented results for CIFAR-100, here we provide accompanying plots in Figure~\ref{fig:nn_cif20_nmvec_auc} and \ref{fig:nn_cif20_genvec_auc} for CIFAR-20, which demonstrates the same trends as those observed in CIFAR-100. We do note that the AUC curves for CIFAR-20 are slightly lower than the respective CIFAR-100 curves. An expected result due to the reduction in the number of class labels.

    

\begin{figure}[t]
    \vspace{-2mm}
    \centering
    \subfloat[Original vectors\label{fig:nn_cif20_nmvec_auc}]{%
    \includegraphics[width=0.48\columnwidth]{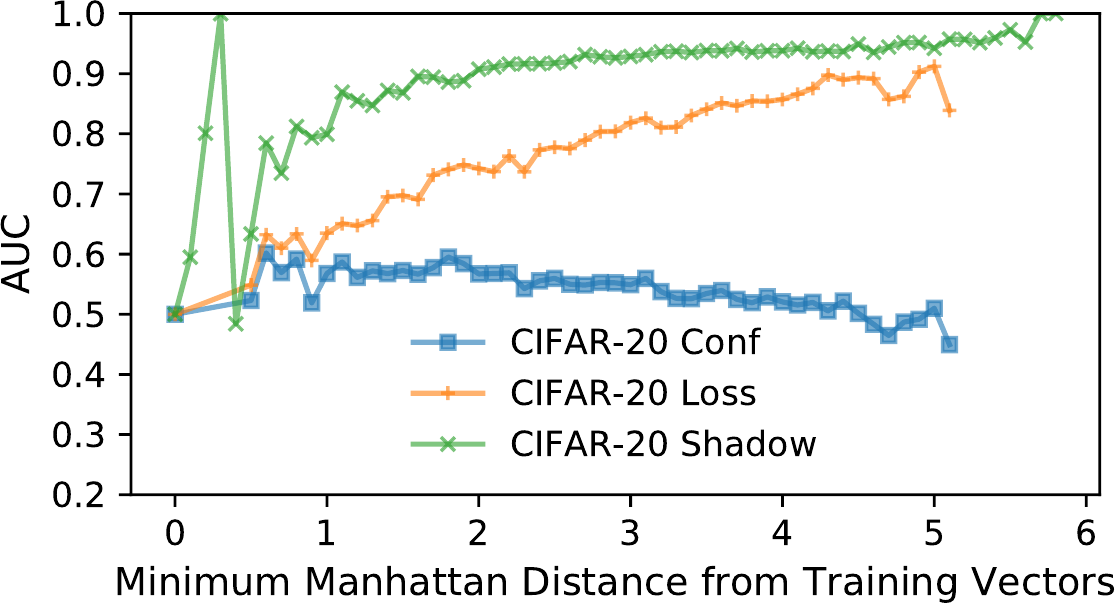}
    } 
    \subfloat[Generated vectors\label{fig:nn_cif20_genvec_auc}]{%
    \includegraphics[width=0.48\columnwidth]{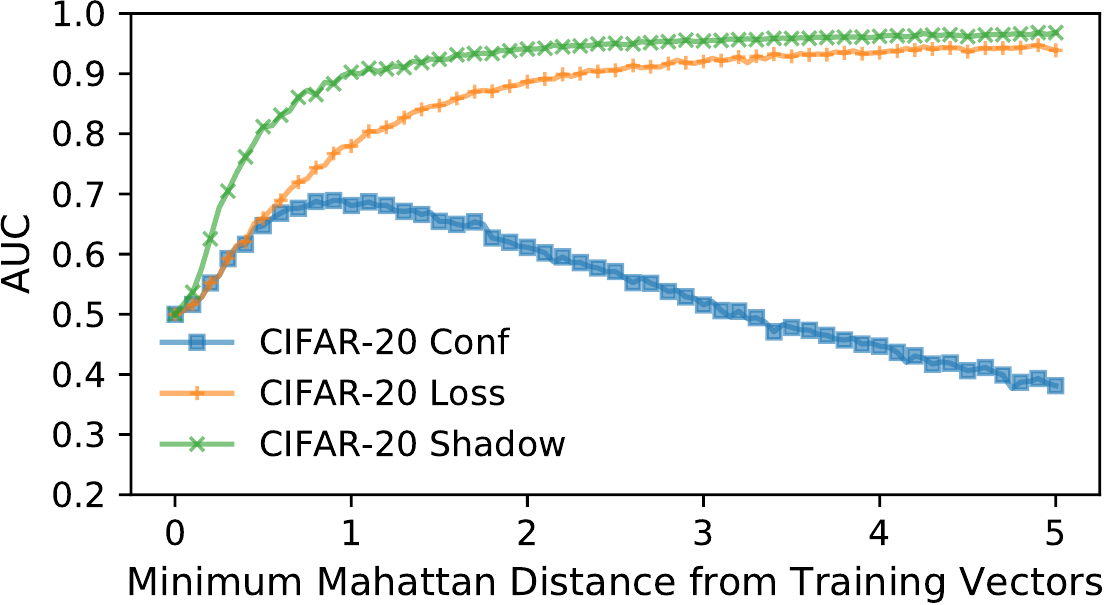}
    }
    \caption{AUC of MI attacks on original and synthetic non-member vectors of the CIFAR-20 dataset as a function of Manhattan distance.}
    \label{fig:nn_cif20_auc}
\end{figure}

\textbf{Per-Label Plots}
\label{sec:appendix-perlabel}
As previously discussed in Section~\ref{sec:mia-labels}, we had only shown the Purchase-20 dataset. We now provide the per-label plots of our remaining binary datasets in Fig.~\ref{fig:mia_label_auc-all}.

\blue{
\subsection{Validating the Indistinguishable Neighbor Assumption}
\label{app:valid}
To demonstrate that the indistinguishable neighbor assumption from Definition~\ref{def:smooth-dist} holds for real-world datasets, we train a Generative Adversarial Network (GAN) to produce and discriminate between real and perturbed vectors from the Purchase dataset. We train the GAN over 50 epochs with 90\% of the data, and evaluate with the remaining 10\%. We use a 100 length noise input to the generator. In Figure~\ref{fig:indist_neighbor}, it is clear that at small distances ($r$-values) there is little advantage in distinguishing between a real vector and a perturbed vector. The advantage increases, and becomes significant, as the distance increases, validating our theoretical assumption.

   \begin{figure}[t]
   \centering
   \includegraphics[width=0.65\columnwidth]{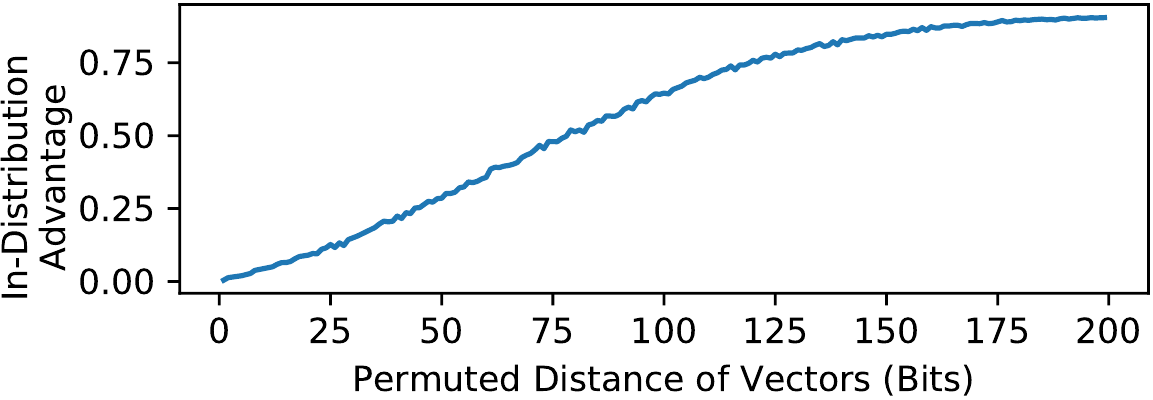} \caption{\blue{Advantage of the GAN distinguisher in distinguishing between real and perturbed vectors from the Purchase dataset at increasing distances.}}
   \label{fig:indist_neighbor}
  \vspace{-3mm}
   \end{figure}
}

\subsection{Exact AI on a Single Missing feature}\label{sec:1feat_ai}
In this section we present an equivalent AI attack to that in Section~\ref{sec:ai-result}, with the exception that only the single most informative feature is to be inferred. Compared to Table~\ref{tab:ai_adv}, we see that AI advantages for a single missing feature are better than their counterparts for multiple missing features. This is intuitively clear since with more feature information withheld from an attacker (15 features as in Section~\ref{sec:ai-result}), the difficulty of the attack increases, and the likelihood of AI success will decrease. However, when compared to Table~\ref{tab:model_traintest}, we note that the significant MI performance (in terms of AUC) is not reflected in the AI performance of Table~\ref{tab:1feat_ai}. For a single missing feature, AI is equivalent to AAI, since in a binary dataset, with only one missing feature, it is either correct or incorrect. Thus, we only perform AAI for the case of multiple missing features, as is done in Section~\ref{sec:ai-result}.



\begin{table}[t]
\caption{Attribute Inference (Exp.~\ref{exp:attr-inf-somesh}) Advantage, where the adversary seeks to infer the exact attribute, when a single most informative feature is missing. The results below are normalized when dealing with ties.}
\label{tab:1feat_ai}
\resizebox{\columnwidth}{!}{%
\begin{tabular}{|r|cccccc|}
\hline
\textbf{AI}  & Loc-30 & Pur-2 & Pur-10 & Pur-20 & Pur-50 & Pur-100                          \\ \hline
Salem Advantage  & 0.0700   & 0.0051     & 0.0266      & 0.0396      & 0.0815      & 0.0917       \\
Yeom Advantage   & 0.0581   & 0.0069     & 0.0191      & 0.0294      & 0.0655      & 0.0791       \\
Shokri Advantage & 0.0377   & -0.0057    & 0.0445      & 0.0581      & 0.0318      & 0.0251       \\
\hline
\end{tabular}
}
\vspace{-3mm}
\end{table}

\subsection{Tuning Attack Models for SMI}\label{sec:tune_model}
{
It may be argued that these MI attacks are not specifically trained to distinguish between members and nearby (synthetic) non-members, which may explain their poor performance in SMI. To investigate if we can improve their performance of SMI, we tune the training process of these attack models to further include nearby synthetic non-members. This augmented training process is only applicable to the MI attacks that employ an attack model, i.e., Shadow, Local WB, and Global WB. The other two MI attacks, i.e., Conf and Loss MI, directly inspect the outputs of the target model for their MI decision, and hence tuning the decision based on member and nearby synthetic non-member vectors is not applicable.


To perform this experiment we take the same experimental steps as Section~\ref{sec:mia-synthetic}, select the Shadow MI attack, and augment the tuning step with synthetic non-members generated from both members and non-members of the attack model training set. For each training vector (member or non-member), we generate two synthetic vectors at all Hamming distances up to 10.
These synthetic non-members are then used to update the attack model. 


\begin{figure}
    \centering
    \subfloat[SMI View\label{fig:tune_attack_model_16}]{%
       \includegraphics[width=0.45\columnwidth]{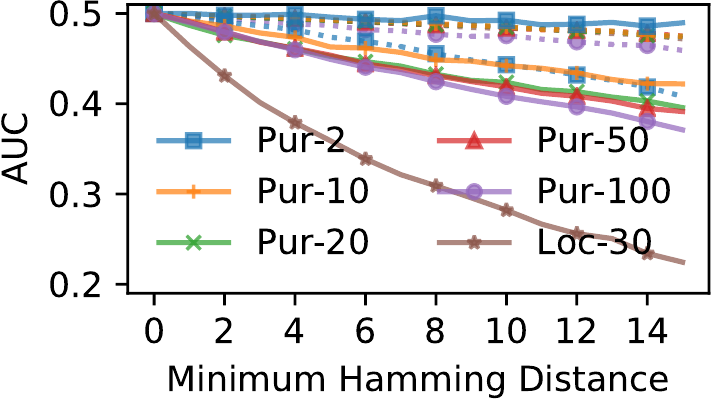}
     }\hfill
     \subfloat[Extended View\label{fig:tune_attack_model_200}]{%
       \includegraphics[width=0.45\columnwidth]{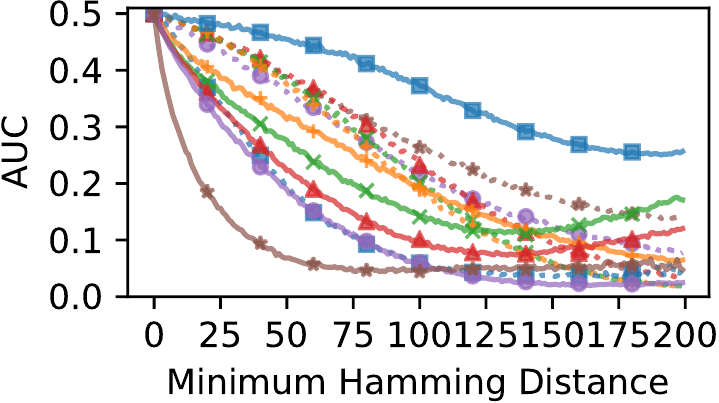}
     }
    \caption{AUC performance on Shadow MI tuned with additional close vectors (dotted lines). The existing Shadow MI results (solid lines) have been mirrored on 0.5 to allow for easier comparison pre and post tuning. 
    }
    \label{fig:tune_attack_model}
    \vspace{-3mm}
\end{figure}

From Fig.~\ref{fig:tune_attack_model}, it can be observed that the AUC of the attack at distances close to the dataset still remains close to 0.5, while at larger distances, the AUC approaches 0, indicating that the attack can distinguish between members and non-members as we move away from the dataset, although with membership label reversed, i.e., more members are now classified as non-members and vice versa. Upon closer inspection, the attack model had no advantage in inferring membership of member vectors (near 0.5 AUC across all datasets). On the other hand, the attack model erred more towards mislabeling non-members (both original and synthetic) as members. We hypothesize this output label `flipping` of the trend is due to the numerous additional close non-members provided to the attack model, which ``confuses'' the model in distinguishing members from non-members, producing an AUC below 0.5.  
Regardless, for all datasets tuning the attack model for SMI does not show any improvement in detecting non-members close to the dataset compared to the original attack model. We also carried out an additional repetition of the experiment with one synthetic vector generated per member and non-member, at each Hamming distance up to 50. This demonstrated worse AUC performance over all distances. 


We conclude that despite the retraining the attack model with additional nearby non-members, the attack failed to achieve SMI. In fact, MI performance generally decreased, due to the similarity of members and the synthetic nearby non-members.  
}

\section{Metrics, Balls and Siblings}
\label{app:metrics}
\blue{The results from Section~\ref{sec:definition} do not apply to any arbitrary distance metric. For instance, given any distance metric $d$, the metric $C \cdot d$, where $C > 0$ is a constant is also a distance metric. But this introduces arbitrarily large (artificial) distance between vectors. We, therefore, restrict ourselves to metrics that do not exhibit arbitrarily large deviation given small perturbation in vectors. This leads to the notion of conserving metric~\cite[\S 1.6]{metric-spaces} to be introduced shortly.}

\begin{theorem}[Metrics]
Let $d_1$ be a metric on $\mathbb{D}$. Let $\mathbf{x}, \mathbf{x}' \in \mathbb{D}^m$. Then the functions
\begin{enumerate}[itemsep=3pt]
    \item $d_M(\mathbf{x}, \mathbf{x}') = \sum_{i=1}^m d_1(x_i, x'_i)$,
    \item $d_E(\mathbf{x}, \mathbf{x}') = \sqrt{\sum_{i=1}^m (d_1(x_i, x'_i))^2}$,
    \item $d_\infty(\mathbf{x}, \mathbf{x}') = \max_{i \in [m]} (d_1(x_i, x'_i))$,
\end{enumerate}
are metrics on the product space $\mathbb{D}^m$. Moreover, for every $\mathbf{x}, \mathbf{x}' \in \mathbb{D}^m$, we have $d_\infty(\mathbf{x}, \mathbf{x}') \le d_E(\mathbf{x}, \mathbf{x}') \le d_M(\mathbf{x}, \mathbf{x}')$ ~\cite[\S 1.6]{metric-spaces}.
\end{theorem}

\begin{definition}[Conserving metric]
A metric $d$ is called a conserving metric~\cite[\S 1.6]{metric-spaces} on the product space $\mathbb{D}^m$ if for all $\mathbf{x}, \mathbf{x}' \in \mathbb{D}^m$, we have 
\[
\pushQED{\qed} 
d_\infty(\mathbf{x}, \mathbf{x}') \le d(\mathbf{x}, \mathbf{x}') \le d_M(\mathbf{x}, \mathbf{x}').
\qedhere
\popQED
\]

\end{definition}
Examples of conserving metrics include the Hamming distance over $\mathbb{D}^m = \{0, 1\}^m$, where $d_1(x, x') = |x-x'|$, $x, x' \in \{0, 1\}$, the Euclidean distance over $\mathbb{D}^m = [0, 1]^m$, where $d_1(x, x') = |x-x'|$, $x, x' \in [0, 1]$, and the Manhattan distance ($d_M$) over $\mathbb{D}^m = [-1, 1]^m$, where $d_1(x, x') = |x-x'|$, $x, x' \in [-1, 1]$. Henceforth we will assume the metric $d$ to be a conserving metric on $\mathbb{D}^m$.

For any subset $X \subseteq \mathbb{D}^m$, the diameter of $X$, denoted $\mathsf{diam}_d(X)$ is defined as $\max \{d(\mathbf{x}, \mathbf{x}') \mid \mathbf{x}, \mathbf{x}' \in X\}$.

\descr{Bounded Feature Space.} We assume $\mathbb{D}$ to be bounded, i.e., $\mathsf{diam}_{d_1}(\mathbb{D}) < \infty$. Since $d$ is a conserving metric it follows that $\mathsf{diam}_d(\mathbb{D}^m) < \infty$, and hence the feature space is also bounded. This is equivalent to saying that for any $\mathbf{x} \in \mathbb{D}^m$, there exists an $R > 0$ such that $\mathbb{D}^m = B_d(\mathbf{x}, R)$~\cite[\S 7.1]{metric-spaces}. 

\descr{Siblings.} Overloading notation, we also define
\[
\Phi_i(\mathbf{x}) = \bigcup_{\substack{S \subseteq [m]\\ |S| = i}} \Phi_S(\mathbf{x}),
\]
where $1 \le i \le m - 1$.

\begin{proposition}
\label{prop:find-r}
Let $1 \le i \le m - 1$. Let $r \ge i \times \mathsf{diam}_{d_1}(\mathbb{D})$. Then for every feature vector $\mathbf{x} \in \mathbb{D}^m$, we have $\Phi_i(\mathbf{x}) \subseteq B_d(\mathbf{x}, r)$. 
\end{proposition}

\begin{proof}
Consider any $\mathbf{x}' \in \Phi_i(\mathbf{x})$. Then $\mathbf{x}' \in \Phi_S(\mathbf{x})$, for some $S \subseteq [m]$ where $|S| = i$. Then, as $d$ is a conserving metric,
\begin{align*}
    d(\mathbf{x}, \mathbf{x}') &\le d_M(\mathbf{x}, \mathbf{x}') \le \sum_{j=1}^m d_1(x_j, x'_j) = \sum_{j \in S} d_1(x_j, x'_j) \\
    &\le \sum_{j \in S} \mathsf{diam}_{d_1}(\mathbb{D}) = i \times \mathsf{diam}_{d_1}(\mathbb{D}) \le r. 
\end{align*}
Hence $\mathbf{x}' \in B(\mathbf{x}, r)$.
\end{proof}

For metrics $d_E$ and $d_M$, we define $d_i$ to be the restriction of $d_E$ or $d_M$ to $i$ dimensions in a natural way, where $1 \le i \le m$.

\begin{proposition}
If $\mathsf{diam}_{d_1}(\mathbb{D}) = \delta > 0$, then $\mathsf{diam}_{d_{1}}(\mathbb{D}) < \mathsf{diam}_{d_{2}}(\mathbb{D}^{2}) < \mathsf{diam}_{d_{3}}(\mathbb{D}^{3}) < \cdots$. 
\end{proposition}
\begin{proof}
Consider the metric to be $d_E$. Consider $i = 1$. Then there exist $x, x' \in \mathbb{D}$ such that $\delta = d(x, x')$. Construct the $2$-dimensional vectors $\mathbf{x} = (x, x)$ and $\mathbf{x}' = (x', x')$. Then, 
\begin{align*}
    \mathsf{diam}_{d_2}(\mathbb{D}^2) &\ge \sqrt{(d_1(x, x'))^2 + (d_1(x, x'))^2} \\
    &= \sqrt{2}\delta > \delta = \mathsf{diam}_{d_1}(\mathbb{D}).
\end{align*}
The rest of the proof follows by induction. The case for $d_M$ is similar. 
\end{proof}

\begin{proposition}
Let $1 \le i \le m - 1$. Let $\mathsf{diam}_{d_{i+1}}(\mathbb{D}^{i+1}) > r \ge \mathsf{diam}_{d_i}(\mathbb{D}^i)$, where $d_j$ is $d_E$ restricted to $j$ dimensions. Then, 
\begin{enumerate}
    \item For any feature vector $\mathbf{x} \in \mathbb{D}^m$, we have $\Phi_i(\mathbf{x}) \subseteq B_{d_E}(\mathbf{x}, r)$.
    \item There exists a feature vector $\mathbf{x} \in \mathbb{D}^m$, such that $\Phi_{i+1}(\mathbf{x}) \not\subseteq B_{d_E}(\mathbf{x}, r)$.
\end{enumerate}
Furthermore, the same holds for the metric $d_M$, and $d_j$ being $d_M$ restricted to $j$ dimensions. 
\end{proposition}


\begin{proof}
For part (1), consider any $\mathbf{x}' \in \Phi_i(\mathbf{x})$. Then $\mathbf{x}' \in \Phi_S(\mathbf{x})$, for some $S \subseteq [m]$ where $|S| = i$. Then,
\begin{align*}
    d_E(\mathbf{x}, \mathbf{x}') &= \sqrt{\sum_{j=1}^m (d_1(x_j, x'_j))^2} \\
    &= \sqrt{\sum_{j \in S} (d_1(x_j, x'_j))^2} \le \mathsf{diam}_{d_i}(\mathbb{D}^i) \le r.
\end{align*}
Hence $\mathbf{x}' \in B_{d_E}(\mathbf{x}, r)$. For part (2), let $\delta = \mathsf{diam}_{d_{i+1}}(\mathbb{D}^{i+1})$. Then their exist $(i+1)$-dimensional vectors $\mathbf{x}', \mathbf{x}'' \in \mathbb{D}^{i+1}$ such that $d_{i+1}(\mathbf{x}', \mathbf{x}'') = \delta$. Furthermore, $d_1(x'_j, x''_j) \neq 0$, for all $j \in [i+1]$. Suppose not, and wlog assume that $d_1(x'_{i+1}, x''_{i+1}) = 0$. Then, we can discard the last element from both vectors, and the resulting $i$-dimensional vectors have distance $\delta$ according to $d_i$, which is greater than $\mathsf{diam}_{d_i}(\mathbb{D}^i)$; a contradiction. Now, sample any $(m - i - 1)$-dimensional vector from $\mathbb{D}^{m - i - 1}$ and append it to both $\mathbf{x}'$ and $\mathbf{x}''$. Let us call the resulting vectors $\mathbf{x}_1$ and $\mathbf{x}_2$. Let $S = \{1, 2, \ldots, i+1\}$. Then, $|S| = i + 1$, and $\mathbf{x}_2 \in \Phi_S(\mathbf{x}_1) \subseteq \Phi_{i+1}(\mathbf{x}_1)$, but 
\begin{align*}
    d_E(\mathbf{x}_1, \mathbf{x}_2) &= \sqrt{\sum_{j=1}^m (d_1(x_j, x'_j))^2} \\
    &= \sqrt{\sum_{j \in S} (d_1(x_j, x'_j))^2} = \delta > r.
\end{align*}
Hence $\mathbf{x}_2 \notin B_{d_E}(\mathbf{x}_1, r)$. 

A similar proof holds for the metric $d_M$.
\end{proof}
\begin{corollary}
\label{cor:phi}
Let $i$ and $\mathbf{x}$ be as in the statement of the previous proposition. Define $d_1(x, x') = | x - x'|$ for $x, x' \in \mathbb{D}$. 
\begin{enumerate}
    \item Let $d_H$ be the Hamming distance on $\mathbb{D} = \{0, 1\}^m$. Let $r \ge i$. Then $\Phi_i(\mathbf{x}) \subseteq B_{d_H}(\mathbf{x}, r)$.
    \item Let $d_M$ be the Manhattan distance on $\mathbb{D} = [-1, 1]^m$. Let $r \ge 2i$. Then $\Phi_i(\mathbf{x}) \subseteq B_{d_M}(\mathbf{x}, r)$.
    \item Let $d_E$ be the Euclidean distance on $\mathbb{D} = [-1, 1]^m$. Let $r \ge \sqrt{4i}$. Then $\Phi_i(\mathbf{x}) \subseteq B_{d_E}(\mathbf{x}, r)$.
\end{enumerate}
\end{corollary}
\blue{The above corollary can be used to select an $r$ such that all siblings of a portion are within the $r$-ball. This is used, for instance, by the AI adversary to employ an SMI attack as a subroutine to infer attributes in Section~\ref{sec:definition}.}


\section{Relationship between Inference Notions}
\label{app:inf-relations}

\descr{Proof of Theorem~\ref{the:mi-smi}.}
\begin{proof}
We essentially show that a membership inference (MI) adversary does not imply a strong membership inference (SMI) adversary, i.e., $\text{MI} \centernot\Rightarrow \text{SMI}$. Let $r > 0$ be fixed. Let $k \ge 2$ be a fixed number of labels. Let $N \gg n$. Sample $N$ points from $\mathbb{R}^m$ such that for all pairs of points $\mathbf{x}, \mathbf{x}'$ in this sample, with $\mathbf{x} \neq \mathbf{x}'$, we have $d(\mathbf{x}, \mathbf{x}') > 3r$.\footnote{There can be many such vectors, which can be found using a greedy algorithm~\cite{gil-var-bound}. For instance, if $\mathbb{D}=\{0, 1\}$, $r = 1$, and $d$ is the Hamming distance, then the Gilbert-Varshamov bound states that there are at least 
$2^m / \sum_{i = 0}^3 \binom{m}{i},$
vectors with minimum Hamming distance $>3r = 3$~\cite{gil-var-bound, asymp-gv-bound}.
} 
Let us call this sample $S_1$. For each $\mathbf{x} \in S_1$, assign it an arbitrary label from the $k$ labels and set $c(\mathbf{x})$ to this label. Initialize an empty set $S_2$. Now for each $\mathbf{x} \in S_1$, sample a random point from $B(\mathbf{x}, r) - \{ \mathbf{x} \}$, and add to $S_2$, and assign it the same label as $\mathbf{x}$, i.e., $c(\mathbf{x})$. Let $S = S_1 \cup S_2$. Notice that every vector in $S$ has precisely one $r$-neighbor in $S$. To see this, first note that every vector in $S_1$ is not an $r$-neighbor of any other vector in $S_1$ by construction. Next, we take a vector $\mathbf{x}$ in $S_1$, and see if it has more than one $r$-neighbors in $S_2$. Let $\mathbf{y}$ be the $r$-neighbor guaranteed by construction. Assume now that $\mathbf{w} \in S_2$ different from $\mathbf{y}$ is another $r$-neighbor of $\mathbf{x}$. Let $\mathbf{z} \in S_1$ be the $r$-neighbor of $\mathbf{w}$ in $S_1$ guaranteed by construction. Then,
\[
 d(\mathbf{x}, \mathbf{z}) \le d(\mathbf{x}, \mathbf{w}) + d(\mathbf{w}, \mathbf{z}) 
    \Rightarrow d(\mathbf{x}, \mathbf{z}) \le r + r = 2r,
\]
a contradiction. Next, we will look at vectors in $S_2$. We will check if any vector from $S_2$ has more than one $r$-neighbor in $S_1$. Then, we will check if the vectors in $S_2$ have any $r$-neighbors in $S_2$. This exhausts the cases. 

Let $\mathbf{y}$ be the $r$-neighbor in $S_2$ of some $\mathbf{x} \in S_1$. This is true by construction. Let $\mathbf{z}$ be some other vector in $S_1$. Then, $d(\mathbf{x}, \mathbf{y}) \le r$, and $d(\mathbf{x}, \mathbf{z}) > 3r$. Therefore,
\begin{align*}
    d(\mathbf{x}, \mathbf{z}) &\le d(\mathbf{x}, \mathbf{y}) + d(\mathbf{y}, \mathbf{z}) \\
\Rightarrow 3r &<  d(\mathbf{x}, \mathbf{y}) + d(\mathbf{y}, \mathbf{z}) \\
\Rightarrow 3r &< r + d(\mathbf{y}, \mathbf{z}) 
\Rightarrow 2r < d(\mathbf{y}, \mathbf{z}), 
\end{align*}
hence $\mathbf{y}$ is not an $r$-neighbor of any other $\mathbf{z}$ in $S_1$. Now consider some $\mathbf{w} \in S_2$ not equal to $\mathbf{y}$. Assume to the contrary that $d(\mathbf{y}, \mathbf{w}) \le r$. Let $\mathbf{z}$ be the $r$-neighbor of $\mathbf{w}$ in $S_1$ (again by construction, it should exist). Then,
\begin{align*}
    d(\mathbf{x}, \mathbf{z}) &\le d(\mathbf{x}, \mathbf{y}) + d(\mathbf{y}, \mathbf{w}) + d(\mathbf{y}, \mathbf{z}) \\
    \Rightarrow d(\mathbf{x}, \mathbf{z}) &\le r + r + r = 3r,
\end{align*}
which is a contradiction. 

Let $\mathbb{D}^m = S$. Define the distribution $\mathcal{D}$ as the uniform distribution over $S$. Sample a dataset $X \leftarrow \mathcal{D}^n$. Define a classifier $h_X$ which given a point $\mathbf{x}$ in $X$, assigns its label $c(\mathbf{x})$ to all vectors within the ball $B(\mathbf{x}, r)$, i.e., all $r$-neighbors of $\mathbf{x}$ have the constant label. The classifier $h_X$, when queried for a point $\mathbf{x} \in X$, simply outputs the label $c(\mathbf{x})$. For any point $\mathbf{x} \notin X$, it checks if there is some $\mathbf{x}' \in X$ such that $d(\mathbf{x}', \mathbf{x}) \le r$. If yes, it returns the label $c(\mathbf{x}')$. Otherwise, it returns an arbitrary label from the $k$ labels. 

Now consider an MI adversary $\adversary$ which given $(\mathbf{x}, c(\mathbf{x}))$, queries $h_X$ with $\mathbf{x}$, and outputs 1 (member) if $h_X(\mathbf{x}) = c(\mathbf{x})$ and 0 (non-member) otherwise. Let us calculate the probabilities in:
\[
\Pr [b' = 1 \mid b = 1] - \Pr [b' = 1 \mid b = 0],
\]
which define the adversary's advantage (Definition~\ref{def:mi-adv}). If $\mathbf{x}$ is a member, then the adversary does not make a mistake, as the label returned by $h_X$ is exactly the label $c(\mathbf{x})$ by construction. Therefore,
\[
\Pr [b' = 1 \mid b = 1] = 1.
\]
Now consider the other probability, i.e., $\Pr [b' = 1 \mid b = 0]$. The adversary could erroneously output $\mathbf{x}$ as a member either if its $r$-neighbor was in $X$, or if its $r$-neighbor was not part of $X$, but the classifier gives it the correct label by chance. Thus
\begin{align*}
    \Pr [b' = 1 \mid b = 0] &= \left( 1 - \left(\frac{2N-2}{2N-1}\right)^n \right) \\
    &+ \left(\frac{2N-2}{2N-1}\right)^n \left( \frac{1}{k}\right) \\
    &= 1 - \left(1 - \frac{1}{2N-1}\right)^n \left(\frac{k-1}{k}\right)
\end{align*}
Subtracting this from the above, we see that the advantage is
\[
\left(1 - \frac{1}{2N-1}\right)^n \left(\frac{k-1}{k}\right)
\]
By Bernoulli's inequality~\cite{ber-ineq}, we have
\[
\left(1 - \frac{1}{2N-1}\right)^n  \ge 1 - \frac{n}{2N-1},
\]
and noting that $N > n$, we get $2N - 1 \ge 2n$. And therefore,
\[
1 - \frac{n}{2N-1} \ge 1 - \frac{n}{2n} = \frac{1}{2}.
\]
Finally, we get the advantage of at least $\frac{1}{2} \frac{k-1}{k}$, which is a constant.\footnote{Note that if the adversary just guesses randomly, the advantage is 0. This is significantly greater than 0.} However, the same adversary if used as a subroutine in Experiment~\ref{exp:strong-mem-inf}, will always output 1 if queried on $\mathbf{x}$ and its $r$-neighbor, since every $r$-neighbor of a member $\mathbf{x} \in X$, is assigned the true label (even if it is not in $X$, by construction). Hence, the resulting adversary has no advantage in the sense of SMI.  
\end{proof}

\section{Miscellaneous Results}
\label{app:misc}

\descr{Relationship between AUC and Advantage.} The MI advantage from Definition~\ref{def:mi-adv} denoted
$\text{Adv}_{\text{MI}}(\adversary, h_X, n, \mathcal{D}) $ can be empirically estimated as $\text{TPR}(\tau) - \text{FPR}(\tau)$\footnote{i.e., $\Pr [b' = 1 \mid b = 1] = \frac{\Pr [b' = 1 \wedge b = 1]}{\Pr [b = 1]}  
=\text{TPR}$ and $\Pr [b' = 1 \mid b = 0] = \frac{\Pr [b' = 1 \wedge b = 0]}{\Pr [b = 0]} 
=\text{FPR}$} with $\tau$ denoting the threshold parameter of the given classifier $h_X$ and $\text{TPR}(\tau)$ and $\text{FPR}(\tau)$ denoting the True Positive Rate and False Positive Rate respectively at $\tau$. The AUC-ROC statistic captures the aggregate performance of the classifier $h_X$ for all possible values of the threshold $\tau$ and is computed as $\text{AUC}=\int_{\text{FPR}(\tau)=0}^1 \text{TPR}(\tau) d(\text{FPR}(\tau)) = \int_{x=0}^1 \text{TPR}(\text{FPR}^{-1}(x)) dx$. 

When $ \text{Adv}_{\text{MI}}(\adversary, h_X, n, \mathcal{D})) = \text{Adv}_m$ for all possible values of $\tau$ (\textit{i.e.} Advantage is same for all values of the threshold parameter), the $\text{AUC}$ is computed as $ \int_{x=0}^1  (\text{FPR}(\text{FPR}^{-1}(x)) + \text{Adv}_m) dx = \frac{1}{2} + \text{Adv}_m$. Thus, $\text{AUC} - \frac{1}{2}$ equals the advantage from Definition~\ref{def:mi-adv}. Even when the advantages vary with $\tau$, $\text{AUC}-\frac{1}{2}$ is a good approximation for the average advantage.

Similarly, the Advantage in the strong membership inference definition, $\text{Adv}_{\text{SMI}}(\adversary, h_X, r, n, \mathcal{D})$ can be empirically estimated as $\text{TPR}(\tau) - \text{FPR}(\tau)$ as long as $B_d(\mathbf{x}_0, r)$ is assumed to have a small number of samples from $X$, i.e., in general $B_d(\mathbf{x}_0, r)$ would contain more elements outside of $X$.

\descr{Average Manhattan Distance.} Let $\mathbb{D}^m = [-1, 1]^m$. Given a vector $\mathbf{x} \in \mathbb{D}^m$, we want to find the Manhattan distance $d_M$ between $\mathbf{x}$ and a vector $\mathbf{y} \in \mathbb{D}^m$, each of whose elements is sampled uniformly at random from the set $\mathbb{D} = [-1, 1]$. Define the distance as $\alpha_m$. Consider first $m = 1$. Then, $\alpha_1$, the expected Manhattan distance between $x$ and $y$, can be defined as
\[
\alpha_1 = \frac{1}{R} \int_{-1}^{+1} \int_{-1}^{+1} |x - y|\, dx\, dy,
\]
where $R = 4$ is the area of the square $[-1, 1] \times [-1, 1]$. Integrating the above we get,
\begin{align*}
    \alpha_1 &= \frac{1}{4} \int_{-1}^{+1} \left(\int_{-1}^{y} (y-x)\, dx + \int_{y}^{+1} (x-y)\, dx \right) \, dy \\
            &= \frac{1}{4} \int_{-1}^{+1} (y^2 + 1) \, dy 
            = \frac{1}{4} \cdot \frac{8}{3} = \frac{2}{3}.
\end{align*}
By independence, we get $\alpha_m = m \alpha_1 = 2m/3$. For $m = 5$, we get $\alpha_5 = 10/3 \approx 3.33$. Thus, we set $\alpha = 3.33$ as the benchmark for a random guess with 5 missing features in the CIFAR dataset.


\end{document}